\documentclass{jair}
\setcopyright{cc}
\acmDOI{}

\JAIRTrack{}            %
\acmVolume{0}
\acmArticle{}
\acmMonth{0}
\acmYear{2026}

\RequirePackage[
  datamodel=acmdatamodel,
  style=acmauthoryear,
  backend=biber,
  giveninits=true,
  uniquename=init
]{biblatex}

\usepackage{amsmath}
\usepackage{amsthm}

\usepackage{amssymb}
\usepackage{tcolorbox}
\usepackage{thm-restate}

\usepackage{url}
\usepackage{algorithm}
\usepackage[noend]{algpseudocode}

\usepackage{subcaption}
\usepackage{xspace}

\usepackage[english]{babel}
\usepackage{color}
\usepackage{booktabs}
\usepackage{enumitem}

\newcommand*\Let[2]{\State #1 $\gets$ #2}

\newtheorem{lemma}{Lemma}
\newtheorem{definition}{Definition}

\newtheorem{theorem}{Theorem}

\newcommand{\Prob}{\mathbb{P}}

\newcommand{\set}[1]{\left\{ #1 \right\}}

\newcommand{\ecoS}{\textsc{Eco Search}\xspace}
\newcommand{\beeS}{\textsc{Bee Search}\xspace}
\newcommand{\heapS}{\textsc{Heap Search}\xspace}

\newcommand{\heap}{\textsc{Heap}}
\newcommand{\pop}{\textsc{pop}}
\newcommand{\peek}{\textsc{peek}}
\newcommand{\seen}{\textsc{Seen}}
\newcommand{\program}{\textsc{p}}
\newcommand{\programbis}{\textsc{q}}
\newcommand{\cost}{\textsc{cost}}
\renewcommand{\succ}{\textsc{Next}}
\newcommand{\costsucc}{\textsc{next-cost}}
\newcommand{\ComputeSuccessor}{\textsc{ComputeNext}}
\newcommand{\minprogram}{\textsc{MinP}}
\newcommand{\mincost}{\textsc{MinCost}}
\newcommand{\myinsert}{\textsc{insert}}

\newcommand{\generated}{\textsc{Generated}}
\newcommand{\indextocost}{\textsc{Index2Cost}}
\newcommand{\queue}{\textsc{Queue}}
\newcommand{\done}{\textsc{Done}}
\newcommand{\pending}{\textsc{Pending}}

\newcommand{\ntstr}{\mathsf{Str}}
\newcommand{\ntint}{\mathsf{Int}}

\newcommand{\crossbeam}{\textsc{Cross\-Beam}}
\newcommand{\euphony}{\textsc{Euphony}}
\newcommand{\deepcoder}{DeepCoder}
\newcommand{\bustle}{\textsc{Bustle}}
\newcommand{\tfcoder}{TF-Coder}

\graphicspath{{./}}

\title[\ecoS: Constant-Delay Best-First Search]{\ecoS: A Constant-Delay Best-First Search Algorithm for Program Synthesis}

\author{Th{\'e}o Matricon}
\email{theomatricon@gmail.com}
\orcid{0000-0002-5043-3221}
\affiliation{%
  \institution{CNRS, LaBRI and Universit{\'e} de Bordeaux}
  \city{Bordeaux}
  \country{France}
}

\author{Nathana{\"e}l Fijalkow}
\email{nathanaelfijalkow@gmail.com}
\orcid{0000-0002-6576-4680}
\affiliation{%
  \institution{CNRS, LaBRI and Universit{\'e} de Bordeaux}
  \city{Bordeaux}
  \country{France}
}

\author{Guillaume Lagarde}
\email{guillaume.lagarde@labri.fr}
\affiliation{%
  \institution{Universit{\'e} de Bordeaux}
  \city{Bordeaux}
  \country{France}
}

\renewcommand{\shortauthors}{Matricon, Fijalkow \& Lagarde}

\begin{abstract}
Many approaches to program synthesis perform a combinatorial search within a
large space of programs to find one that satisfies a given specification.
To tame the search space blowup, previous works introduced probabilistic and
neural approaches to guide this combinatorial search by inducing heuristic cost
functions.
Best-first search algorithms ensure to search in the exact order induced by the
cost function, significantly reducing the portion of the program space to be
explored.

A key limitation of existing best-first search algorithms is that they
\emph{slow down over time}: the delay between outputting successive programs
grows with the number of programs already generated.
We ask whether best-first search algorithms with \emph{constant delay} exist.

We design \ecoS, a new best-first bottom-up search algorithm.
It combines the cost tuple representation of \textsc{Bee Search} with the
per non-terminal data structures of \textsc{Heap Search}, thanks to a new
frugal expansion strategy, and it uses bucket queues instead of priority
queues, which is possible because the costs stored in each queue always lie in
a window of constant width.

We prove that \ecoS is the first constant-delay best-first search algorithm for
pre-generation cost functions.
Experimentally, we evaluate \ecoS on two classic program-synthesis benchmarks
(DeepCoder and FlashFill) and observe that \ecoS solves $12\%$ to $70\%$ more
tasks than its predecessors, and reaches the number of tasks they solve $1.6$ to
$3.2$ times faster.

This shows that constant delay can be achieved for best-first search in
program synthesis, and that it makes a difference in practice.
\end{abstract}

\makeatletter
\def\@formatdoi#1{}
\AtBeginDocument{\fancyfoot{}\fancyfoot[RO,LE]{\footnotesize Preprint, under review at the \@journalName.}%
  \fancypagestyle{firstpagestyle}{\fancyhf{}\fancyfoot[RO,LE]{\footnotesize Preprint, under review at the \@journalName.}}}
\makeatother
\begin{document}
\maketitle

\section{Introduction}
Program synthesis is one of the oldest dream of Artificial Intelligence: it automates problem solving by generating a program meeting a given specification~\cite{Manna.Waldinger:1971,Gulwani.Polozov.ea:2017}.
A very classical scenario for user-based program synthesis, known as programming by example (PBE), uses input output examples as specification.
\begin{sloppypar}
For PBE, combinatorial search for program synthesis has been an especially popular technique~\cite{Alur.Radhakrishna.ea:2017,Balog.Gaunt.ea:2017,Alur.Singh.ea:2018,Shi.Steinhardt.ea:2019,Barke.Peleg.ea:2020,Zohar.Wolf:2018,Ellis.Wong.ea:2021,Odena.Shi.ea:2021,Fijalkow.Lagarde.ea:2022,Shi.Bieber.ea:2022,Shi.Dai.ea:2022,Ameen.Lelis:2023}.
\end{sloppypar}

To scale combinatorial search for program synthesis, many approaches rely on defining a heuristic cost function assigning to every program a numerical value, such that the programs with least scores are the most likely to satisfy the specification.
For example, \deepcoder{}~\cite{Balog.Gaunt.ea:2017} and \tfcoder{}~\cite{Shi.Bieber.ea:2022} use neural models, while \bustle~\cite{Odena.Shi.ea:2021} leverages probabilistic methods for defining a heuristic cost function.
Very recently, LLMs have been used for guiding combinatorial search~\cite{Li.Parsert.ea:2024,Li.Ellis:2024}.

Best-first search algorithms explore the space in the exact order induced by the cost function, which reduces the part of the program space that needs to be explored.
Since \euphony{}~\cite{Alur.Radhakrishna.ea:2017}'s use of $A^*$ algorithm, several best-first search algorithms have been constructed~\cite{Shi.Bieber.ea:2022,Ellis.Wong.ea:2021,Fijalkow.Lagarde.ea:2022,Ameen.Lelis:2023}.

The main issue with best-first search algorithms is that they \emph{slow down over time}: to guarantee that programs come out in the right order, they keep a growing frontier of candidate programs in their data structures, and this frontier quickly becomes very large.
The notion of \emph{delay} captures this behaviour: it quantifies the amount of compute required between outputting two programs. 
The first best-first search algorithm had linear delay~\cite{Alur.Radhakrishna.ea:2017}, and the state of the art algorithms achieve logarithmic delay~\cite{Fijalkow.Lagarde.ea:2022,Ameen.Lelis:2023}: the compute required between outputting the $t$\textsuperscript{th} and the $(t+1)$\textsuperscript{th} program is bounded by $O(\log(t))$.

The fundamental question explored in this paper is whether \textbf{there exist best-first search algorithms with constant delay}. We answer this question positively by constructing the first constant-delay best-first search algorithm called \ecoS for pre-generation cost function.
\ecoS performs a \emph{bottom-up search}, so it can use the classical observational equivalence techniques. 

\ecoS builds on two previous algorithms.
From \beeS~\cite{Ameen.Lelis:2023} it takes the ``cost tuple representation'', which represents sets of programs of the same cost without constructing them.
From \heapS~\cite{Fijalkow.Lagarde.ea:2022} it takes the idea of maintaining separate data structures for each non-terminal of the grammar.
Combining the two is not straightforward: with one list of costs per non-terminal, the algorithm needs to know the cost of programs that have not been generated yet.
Our first contribution is \emph{frugal expansion}, which obtains these costs by reading them from the queues, without generating the corresponding programs.
Our second contribution is a structural property of pre-generation cost functions: we show that the costs stored in each queue always lie in a window of constant width.
This allows us to replace priority queues by bucket queues, a classical data structure for shortest paths with integer weights~\cite{Dial:1969}, which supports operations in constant time.
Each ingredient taken alone is simple, and two of them come from previous work; the constant delay comes from their combination, and we explain in Section~\ref{sec:eco_ablation} why none of them can be dropped.

We demonstrate the effectiveness of \ecoS in two classic domains: 
the \deepcoder{}~\cite{Balog.Gaunt.ea:2017} domain of integer list manipulations and in the FlashFill~\cite{Gulwani:2011} domain of string manipulations.
In our experiments, \ecoS solves $12\%$ to $70\%$ more tasks than previous best-first search algorithms, and reaches the number of tasks they solve $1.6$ to $3.2$ times faster.
To summarize, our contributions are the following:
\begin{itemize}
   \item We introduce \ecoS, a new best-first bottom-up search algorithm, which combines ideas of \heapS and \beeS thanks to frugal expansion;
   \item We prove that the costs in its queues lie in a bounded window, so that bucket queues can be used, and that \ecoS is complete and has constant delay; 
   \item All the lemmas and theorems of this paper, including those on \heapS, have been checked in the Lean proof assistant; the formalization is available as ancillary files of this arXiv version. In the process we found that the original pseudocode of \heapS~\cite{Fijalkow.Lagarde.ea:2022} is not correct for grammars with cycles, and we present a corrected version in Section~\ref{sec:heap_search};
   \item Experimentally, we observe that \ecoS provides significant improvements over existing algorithms.
\end{itemize}

\paragraph*{Relation to the conference version.}
A preliminary version of this work appeared at AAAI 2025~\cite{Matricon.Fijalkow.ea:2025}.
The present article adds complete proofs of correctness, completeness and constant delay, the relation between our terminology and the one of heuristic search, a detailed presentation of \heapS and \beeS with a corrected version of \heapS, a discussion of the ingredients of \ecoS and of observational equivalence and post-generation cost functions, a formalization of all the results in Lean, and new experiments, including ablations and memory measurements.

\paragraph*{Paper structure.}
Section~\ref{sec:background} introduces the formal setup: cost-guided combinatorial search, context-free grammars, pre-generation cost functions, and the relation between our terminology and the one of heuristic search.
Section~\ref{sec:predecessor_algorithms} describes the predecessor algorithms \heapS and \beeS that \ecoS builds upon.
Section~\ref{sec:enumeration_algorithms} presents the \ecoS algorithm, proves that it is complete and best-first, and establishes the constant-delay guarantee.
Section~\ref{sec:experiments} describes our experimental evaluation.
Section~\ref{sec:related_works} surveys related work, and Section~\ref{sec:conclusions} concludes with open problems.
 
\section{Background}
\label{sec:background}
\subsection{Cost-guided combinatorial search}

We consider a set $P$ of elements, which for our applications in program synthesis is the class of all programs.
Given a specification $\varphi$ we write $p \models \varphi$ when the program $p$ satisfies the specification $\varphi$: we say that $p$ is a solution program.
Note that this definition is independent of the type of specification: a logical formula, a set of input output examples, or any other type of specifications discriminating between solutions and not solutions.

The goal of \textit{combinatorial search} for program synthesis is given a specification~$\varphi$ to find a solution program.
Sometimes it is useful to find more than one solution program, or even all of them; in this paper we focus on finding a single one, but the algorithms naturally extend to finding a finite number of solution programs.

In \textit{cost-guided combinatorial search}, we further assume a cost function $w : P \to \mathbb{R}_{> 0}$, mapping each program to a (positive) cost.
The cost function $w$ is used as a heuristic: the smaller the cost of a program, the more likely it is to be a solution program.
In this context, \textit{best-first search algorithms} enumerate programs by increasing costs.
We say that such an algorithm is \emph{complete} if every program is eventually output.

\subsection{Domain-specific languages and context-free grammars}

Let us now be more specific about how programs are represented.
A domain-specific language (DSL) is a programming language designed to solve a specific set of tasks.
Classically, we represent DSLs using context-free grammars (CFGs), and more precisely deterministic tree grammars.
We let $\Sigma$ denote the set of primitive symbols, which include variables. Each symbol has a fixed arity (variables and constants have arity $0$).
The set of non-terminal symbols is $\Gamma$, and $S \in \Gamma$ is the initial non-terminal.
Derivation rules have the following syntax:
\[
X \rightarrow f(X_1,\dots,X_k),
\]
where $f \in \Sigma$ has arity $k$ and $X,X_1,\dots,X_k \in \Gamma$.
When $f$ has arity $0$ (a constant or a variable), we write the rule $X \rightarrow a$ with $a = f$, and call it a \emph{terminal rule}.
The CFG is deterministic if given $X$ and $f$, there is a unique derivation rule from $X$ with $f$.
A grammar acts as a generator: it generates trees, which we call programs.
We say that a program is generated \emph{from} $X$ if it can be derived starting from $X$.
Since the grammar is deterministic, a program generated from $X$ has a unique derivation tree, so we can speak of the non-terminal labelling each of its nodes.
We assume without loss of generality that every non-terminal generates at least one program (non-terminals that do not can be removed beforehand).

\begin{figure}
\begin{center}
\begin{tabular}{lllll}
$r_1$ : & \verb+str+ & $\rightarrow$ & ``Hello" & cost: $1.1$\\
$r_2$ : & \verb+str+ & $\rightarrow$ & ``World" & cost: $2.0$\\
$r_3$ : & \verb+str+ & $\rightarrow$ & \verb+cast+(\verb+int+) & cost: $4.4$ \\
$r_4$ : & \verb+str+ & $\rightarrow$ & \verb+concat+(\verb+str+, \verb+str+) & cost: $5.3$ \\
$r_5$ : & \verb+int+ & $\rightarrow$ & \verb+var+ & cost: $1.8$\\
$r_6$ : & \verb+int+ & $\rightarrow$ & \verb+1+ & cost: $3.3$\\
$r_7$ : & \verb+int+ & $\rightarrow$ & \verb+add+(\verb+int+, \verb+int+) & cost: $5.3$
\end{tabular}
\end{center}
\caption{A simple DSL. The costs of the derivation rules were chosen by hand for the purpose of illustration.}
\label{fig:running_example_grammar}
\end{figure}

\vskip1em
To make things concrete, let us consider a small example.
Our DSL manipulates strings and integers, hence it uses two types: \verb+string+ and \verb+int+.
It has three primitives:

\begin{tabular}{l}
\verb+cast: int -> string+ \\
\verb+concat: string -> string -> string+ \\
\verb+add: int -> int -> int+
\end{tabular}

Let us add constants \verb+"Hello", "World": string+ and \verb+1: int+.
We also add a variable \verb+var: int+.
The class of programs of type \verb+int -> string+ is generated by the CFG given in Figure~\ref{fig:running_example_grammar}, which uses two non-terminals, \verb+string+ and \verb+int+, with the former being initial.
An example program generated by this grammar is
\verb+concat("Hello", cast(add(var,1)))+.
Using the natural semantics for \verb+concat+, \verb+cast+, and \verb+1+, this program concatenates \verb+"Hello"+ to the result of adding $1$ to the input variable and casting it as a string.

\subsection{Pre-generation cost functions}

In most cases cost functions are of a special nature: they are computed recursively alongside the grammar
and induced by defining $\cost(r)$ for each derivation rule $r$ (see Figure~\ref{fig:running_example_grammar} for an example).
Note that $\cost(r)$ can be any positive real number.
Consider a program $\program = f(\program_1,\dots,\program_k)$ generated by the derivation rule
$r: X \rightarrow f(X_1,\dots,X_k)$, meaning that $\program_i$ is generated by $X_i$, then

\[
\cost(\program) = \cost(r) + \sum_{i = 1}^k \cost(\program_i).
\]

What makes pre-generation cost functions special is that they do not depend on executions of the programs, in fact they do not even require holding the whole program in memory since they are naturally computed recursively.
Pre-generation cost functions is a common assumption~\cite{Balog.Gaunt.ea:2017, Ellis.Wong.ea:2021, Fijalkow.Lagarde.ea:2022}.

The most common way of obtaining such costs is through a probabilistic context-free grammar (PCFG), which assigns a probability $\Prob(r)$ to each derivation rule $r$ such that the probabilities of the rules from each non-terminal sum to $1$.
The probability of a program is the product of the probabilities of the rules in its derivation tree.
Setting $\cost(r) = -\log \Prob(r)$ gives $\cost(\program) = -\log \Prob(\program)$, so that enumerating programs by increasing costs amounts to enumerating them by decreasing probabilities.
In our experiments, as in~\cite{Balog.Gaunt.ea:2017,Fijalkow.Lagarde.ea:2022}, the probabilities $\Prob(r)$ are predicted by a neural network trained on synthetic tasks (Section~\ref{sec:experiments}).
The costs in Figure~\ref{fig:running_example_grammar} do not come from such a model, they were chosen to make the examples interesting.

Two simple consequences of the definition will be used repeatedly.
First, every program has a cost larger than or equal to its size multiplied by the smallest cost of a derivation rule, hence for any bound $B$ there are finitely many programs of cost at most $B$.
Second, the cost of a program is strictly larger than the cost of each of its sub-programs.

\subsection{Observational equivalence}

We say that two programs $p$ and $p'$ are \emph{observationally equivalent} with respect to a set of input-output examples $\{(x_1, y_1), \ldots, (x_n, y_n)\}$ if they produce the same output on every input: $p(x_i) = p'(x_i)$ for all $i \in \{1,\ldots,n\}$.
Bottom-up search algorithms exploit this relation to prune the search space: whenever a newly generated program is observationally equivalent to a program already stored for the same non-terminal, the newer copy is discarded and never used as a sub-program.
Since programs are generated by non-decreasing costs, the newer copy has a larger or equal cost.
The comparison must be restricted to programs from the same non-terminal: a program from $X$ that is equivalent to a stored program from another non-terminal cannot be replaced by it.
This pruning does not lose solutions, in the following sense.
Since the output of $f(\program_1,\dots,\program_k)$ on the examples only depends on the outputs of the $\program_i$'s on the examples, replacing a sub-program by an equivalent one of smaller or equal cost gives an equivalent program of smaller or equal cost.
It follows by induction that for every program $\program$, some program equivalent to $\program$ and of cost at most $\cost(\program)$ is kept.
In particular, if $\program$ is a solution (for a specification given by these examples), then a solution of cost at most $\cost(\program)$ is found.
We call this property \emph{completeness up to observational equivalence}.
Observational equivalence is particularly effective in bottom-up algorithms, since sub-programs are generated and evaluated before full programs are assembled, allowing early pruning of redundant sub-trees.

\subsection{Delay}
\label{sec:delay}

The algorithms we study are enumeration algorithms: they do not terminate, and output programs one after the other.
We measure their efficiency by their \emph{delay}, the amount of computation performed between two consecutive programs~\cite{Strozecki:2019}.
For bottom-up algorithms we count every program that the algorithm constructs and stores, whatever its non-terminal: in program synthesis from examples, each of these programs is evaluated on the examples (in particular to check observational equivalence), so it is a natural unit of work.
We use the unit-cost RAM model, and represent programs with shared sub-programs: a program $f(\program_1,\dots,\program_k)$ is a node labelled by $f$ with pointers to $\program_1,\dots,\program_k$, so that constructing it takes time $O(k)$ once the $\program_i$'s are available.
An algorithm has \emph{constant delay} if its delay is bounded by a constant that depends on the grammar and the cost function, but not on the number of programs already output.

\subsection{Relation to the terminology of heuristic search}
\label{sec:terminology}

Readers familiar with heuristic search may find some of our terms non-standard, so let us relate them to the classical ones (open list, closed list, generation and expansion of nodes, as in $A^*$).
The algorithms of this paper all perform a best-first search in an implicit graph, but the nodes of this graph are not partial programs, and the goal is not to reach a single goal node: the algorithms enumerate all programs in order and the specification is checked on each of them.
In \heapS the nodes are programs, while in \beeS and \ecoS the nodes are cost tuples, which represent sets of programs of the same cost (see Section~\ref{sec:bee_search}).
Table~\ref{tab:terminology} gives the correspondence.
In particular, the \emph{next program} of a program $\program$ (Section~\ref{sec:heap_search}) is not a child of $\program$ in the search graph: it is the program that comes right after $\program$ in the cost order.

\begin{table}[h]
\centering
\small
\caption{Correspondence between our terminology and the classical terminology of heuristic search.}
\label{tab:terminology}
\begin{tabular}{lll}
\toprule
Classical term & \heapS & \beeS and \ecoS \\
\midrule
node & program & cost tuple \\
open list & $\heap_X$ (programs in $\seen$ not yet output) & $\queue$ (resp.~$\queue_X$) \\
closed list & programs already output (stored in $\succ_X$) & $\generated$ (resp.~$\generated_X$) \\
generating a node & inserting a program in $\heap_X$ (it is \emph{seen}) & pushing a cost tuple \\
expanding a node & popping a program and inserting its children & popping a cost tuple and pushing its children \\
children of a node & replace one argument by its next program & increase one index by $1$ \\
\bottomrule
\end{tabular}
\end{table}

When we say that an algorithm \emph{generates} or \emph{outputs} a program, we mean that the program is constructed, stored, and given to the rest of the synthesis pipeline (to be checked against the specification); in the classical terminology this happens when a node is expanded.
 
\section{Predecessor Algorithms: \heapS and \beeS}
\label{sec:predecessor_algorithms}
In this section we present in detail two best-first bottom-up search algorithms, \heapS~\cite{Fijalkow.Lagarde.ea:2022} and \beeS~\cite{Ameen.Lelis:2023}.
Bottom-up search starts with the smallest programs and iteratively generates larger programs by combining the smaller ones generated by the algorithm.

To make our pseudocode as readable as possible we use the generator syntax of Python.
In particular, the \textbf{yield} statement is used to return an element (a program in our case) and continue the execution of the code.

We use the DSL presented in Figure~\ref{fig:running_example_grammar} as running example for the algorithms.
For readability, we will use some abbreviations: $\ntstr = \texttt{string}$ and $\ntint = \texttt{int}$ for the two non-terminals, $H = \texttt{"Hello"}$, and $W = \texttt{"World"}$.

\subsection{Computing programs of minimal costs}

As a warm-up, we need a procedure to compute for each non-terminal $X$ a program of minimal cost.
Note that this is well defined because costs are positive, and that we do not require to compute all minimal programs, just a single one.
The pseudocode is given in Algorithm~\ref{algo:minimal_programs}.
The first loop (lines 4 to 7) considers the terminal rules $X \rightarrow a$, that is, the rules producing a constant or a variable $a$.
Throughout the algorithm, $\mincost(X)$ is the cost of the best program found so far from $X$, so it starts at $\infty$ and only decreases; this is unrelated to the cost function itself, for which a program always costs more than its sub-programs.
The algorithm then propagates the minimal programs and costs found across derivation rules, and repeats the propagation as long as it updates values.
The order in which the rules are processed in line 11 does not matter for correctness, because the \textbf{while} loop is repeated until no value changes.
For instance, with the rules $X_1 \rightarrow a$, $X_3 \rightarrow g(X_1,X_2)$ and $X_2 \rightarrow f(X_1)$ processed in this order, the first iteration of the \textbf{while} loop cannot update $X_3$ since $\mincost(X_2) = \infty$, but it updates $X_2$, and the second iteration updates $X_3$.
This is the same principle as in the Bellman-Ford algorithm for shortest paths (see also~\cite{Knuth:1977} for a generalisation of Dijkstra's algorithm to grammars).
Let us show that the number of iterations of the \textbf{while} loop (line 9) is at most $|\Gamma|$.
For each non-terminal $X$ we can choose a program of minimal cost from $X$ in which no non-terminal appears twice on a path from the root to a leaf: otherwise, replacing the sub-tree rooted at the higher occurrence by the one rooted at the lower occurrence gives a program from $X$ of strictly smaller cost.
Such a program has height at most $|\Gamma|$.
By induction on $h$, after $h - 1$ iterations of the \textbf{while} loop, $\mincost(X)$ is correct for every $X$ that has a minimal program of height at most $h$.
Hence all values are correct after $|\Gamma| - 1$ iterations, and one more iteration detects that nothing changes.
In practice the number of iterations is often much smaller.

\begin{algorithm}
\caption{Computing programs of minimal costs}\label{algo:minimal_programs}
\begin{algorithmic}[1]
\For{$X$ non-terminal}
	\Let{$\mincost(X)$}{$\infty$}
	\Let{$\minprogram(X)$}{$\bot$}
\EndFor
\For{$r : X \rightarrow a$ terminal rule}
	\If{$\cost(r) < \mincost(X)$}
		\Let{$\mincost(X)$}{$\cost(r)$}
		\Let{$\minprogram(X)$}{$a$}
	\EndIf
\EndFor
\Statex
\Let{$\text{updated}$}{$\text{True}$}
\While{$\text{updated}$}
	\Let{$\text{updated}$}{$\text{False}$}
	\For{$r : X \rightarrow f(X_1,\dots,X_k)$ derivation rule}
		\Let{$c$}{$\cost(r) + \sum_{i = 1}^k \mincost(X_i)$}
		\If{$c < \mincost(X)$}
			\Let{$\mincost(X)$}{$c$}
			\Let{$\minprogram(X)$}{$f(\minprogram(X_1),\dots,\minprogram(X_k))$}
			\Let{$\text{updated}$}{$\text{True}$}
		\EndIf
	\EndFor
\EndWhile
\end{algorithmic}
\end{algorithm}

\subsection{\heapS}
\label{sec:heap_search}

\heapS was introduced by~\citet{Fijalkow.Lagarde.ea:2022}.
To simplify the presentation, we assume in this section that two different programs from the same non-terminal never have the same cost; in general, ties are broken by a fixed total order, as done by any heap implementation, and all the arguments below go through with this total order in place of the cost order.

Let us fix a non-terminal $X$, and $\program, \program'$ two programs generated from $X$.
We say that $\program'$ is the \emph{next program} of $\program$ (from $X$) if $\cost(\program) < \cost(\program')$
and there does not exist $\program''$ generated from $X$ such that $\cost(\program) < \cost(\program'') < \cost(\program')$.
In other words, $\program'$ has minimal cost among the programs from $X$ of higher cost than $\program$.
The next program is unique, and we also define the next program of a dummy program $\bot$ as the program of minimal cost from $X$.
Since there are finitely many programs of cost at most $B$ for any $B$, the sequence $\bot$, its next program, the next program of that one, and so on, contains every program from $X$, ordered by increasing costs.
Hence an algorithm that computes this sequence for the initial non-terminal $S$ is a complete best-first search algorithm.
Note that the next program of $\program$ is in general not obtained from $\program$ by a derivation step: it is the program that follows $\program$ in the cost order (see Section~\ref{sec:terminology}).

The \heapS algorithm maintains three objects:
\begin{itemize}
	\item $\seen$: stores all programs seen so far. Note that \textit{seen} is not the same as \textit{generated}, as we discuss below.
	\item for each non-terminal $X$, $\heap_X$ is a heap of programs, using as value the costs of the programs. Programs in $\heap_X$ are \textit{seen} but are yet to be \textit{generated}.
	\item for each non-terminal $X$, $\succ_X$ is a mapping from programs to programs: when $\succ_X(\program)$ is defined, it is the next program of $\program$ from $X$.
	Each program from $X$ that has been generated is a key of $\succ_X$ once its own next program has been computed, so the whole sequence of programs from $X$ generated so far is stored in $\succ_X$.
\end{itemize}

Let us explain the difference between \textit{seen} and \textit{generated}.
A program is seen before it is generated.
When a program is inserted (using the function \myinsert), it is seen: it is sitting in some heap waiting for its turn, in the terminology of heuristic search it is in the open list.
It is generated when it is popped from the heap and becomes the value of $\succ_X$ for some key.
It is expanded later, when its own next program is requested: the programs obtained from it by replacing one argument by the next program of this argument are then inserted.
The programs that are yielded in line~\ref{line:hs_yield} of Algorithm~\ref{algo:heap_search_next} are the programs generated from $S$.

\vskip1em
The initialisation is given in Algorithm~\ref{algo:heap_search_init}, and the main function is \ComputeSuccessor in Algorithm~\ref{algo:heap_search_next}: given a program $\program$ generated from $X$, it computes the next program of $\program$ from $X$.
It works as follows: either the next program of $\program$ was already computed (therefore stored in $\succ_X$), in which case it is simply returned, or it was not.
The first case is a form of memoization, and it happens very often: although each program is generated once, the same query $\ComputeSuccessor(\program, X)$ is asked many times, because the program $\program$ is a sub-program of many larger programs and each of them asks for the next program of $\program$.
In the second case, \ComputeSuccessor first expands $\program$ (lines~\ref{line:hs_for} to~\ref{line:hs_endfor}): for each argument $\program_i$ of $\program$, it computes the next program $\program'_i$ of $\program_i$ by a recursive call, and inserts the program obtained from $\program$ by replacing $\program_i$ by $\program'_i$.
It then pops the cheapest program of $\heap_X$, which is the next program of $\program$.
The key invariant of the algorithm is that for each non-terminal $X$, the programs generated from $X$ are the cheapest programs from $X$, and all of them except the last one have been expanded.
We show below that once the last one has been expanded, its next program is the cheapest program in $\heap_X$.
To simplify the presentation we assume that every non-terminal generates infinitely many programs, so that the heaps are never empty; otherwise \ComputeSuccessor returns $\bot$ when $\heap_X$ is empty.
We now state and prove the correctness and complexity of \ComputeSuccessor.

\begin{lemma}[Correctness of \ComputeSuccessor]
\label{lem:heap_search_correctness}
For any non-terminal $X$ and program $\program$ generated from $X$, if $\ComputeSuccessor(\program'', X)$ has already been called for every program $\program''$ preceding $\program$ from $X$ (i.e.\ every program from $X$ of cost strictly less than $\cost(\program)$), then $\ComputeSuccessor(\program, X)$ returns the next program of $\program$ from $X$.
\end{lemma}

\begin{proof}
We prove the lemma together with the following invariants, by induction on the sequence of calls to \ComputeSuccessor.
All recursive calls made by $\ComputeSuccessor(\program, X)$ are on strict sub-programs of $\program$, so they are covered by the induction.
For every non-terminal $Y$:
\begin{itemize}
	\item[(I1)] the programs generated from $Y$ are the first programs from $Y$ in the cost order, and $\succ_Y$ maps $\bot$ to the first one and each of them, except the last one, to the next one;
	\item[(I2)] the arguments of a generated program are generated (from the corresponding non-terminals);
	\item[(I3)] $\heap_Y$ contains exactly the programs from $Y$ that are in $\seen$ and are not generated, and every generated program $\programbis$ from $Y$ such that $\succ_Y(\programbis)$ is defined has been expanded.
\end{itemize}
After the initialisation, the generated programs are the minimal programs $\minprogram(Y)$, popped by the calls $\ComputeSuccessor(\bot, Y)$, and the invariants hold.

Consider a call $\ComputeSuccessor(\program, X)$ where $\program$ is generated from $X$ (or $\program = \bot$).
If $\succ_X(\program)$ is defined, the result is correct by (I1).
Otherwise, by (I1), $\program$ is the last program generated from $X$.
For each argument $\program_i$ of $\program$, which is generated from $X_i$ by (I2), the recursive call returns the next program of $\program_i$ by induction, and the program obtained from $\program$ by replacing $\program_i$ by this next program is inserted into $\heap_X$, unless it is already in $\seen$.
Let $\programbis$ be the next program of $\program$ from $X$.
We show that $\programbis$ is in $\heap_X$ when $\heap_X$ is popped.
\begin{itemize}
	\item $\programbis$ is not generated, since by (I1) all generated programs from $X$ cost at most $\cost(\program) < \cost(\programbis)$.
	\item $\programbis$ is in $\seen$. Write $\programbis = f(\programbis_1,\dots,\programbis_k)$ using the derivation rule $X \to f(X_1,\dots,X_k)$.
	If every $\programbis_i$ is the minimal program $\minprogram(X_i)$, then $\programbis$ was inserted during the initialisation (lines~\ref{line:hs_init_for} to~\ref{line:hs_init_endfor} of Algorithm~\ref{algo:heap_search_init}).
	Otherwise, some $\programbis_j$ has a predecessor $\programbis_j^-$ from $X_j$, that is, $\programbis_j$ is the next program of $\programbis_j^-$.
	Let $\programbis^- = f(\programbis_1,\dots,\programbis_j^-,\dots,\programbis_k)$.
	Since $\cost(\programbis^-) < \cost(\programbis)$ and no program from $X$ has a cost strictly between $\cost(\program)$ and $\cost(\programbis)$, we have $\cost(\programbis^-) \leq \cost(\program)$, so $\programbis^-$ is generated by (I1), and it has been expanded: either $\programbis^- = \program$, which has just been expanded, or $\programbis^-$ comes before $\program$, so $\succ_X(\programbis^-)$ is defined and (I3) applies.
	When $\programbis^-$ was expanded, the call $\ComputeSuccessor(\programbis_j^-, X_j)$ returned $\programbis_j$, so $\programbis$ was inserted, unless it was already in $\seen$.
\end{itemize}
By (I3), $\programbis$ is in $\heap_X$.
Moreover no program in $\heap_X$ is generated, so by (I1) its cost is larger than $\cost(\program)$, hence at least $\cost(\programbis)$.
Since costs are distinct, the pop returns $\programbis$.

It remains to check that the invariants are maintained.
For (I1), $\programbis$ becomes the last generated program from $X$, and $\succ_X(\program) = \programbis$.
For (I2), each argument of $\programbis$ is a minimal program, generated during the initialisation, or an argument of a generated program, or the result of a call to \ComputeSuccessor, so it is generated.
For (I3), $\program$ has been expanded, $\programbis$ leaves $\heap_X$ when it becomes generated, and the inserted programs are not generated, since generated programs are in $\seen$.
\end{proof}

\begin{lemma}[Complexity of \ComputeSuccessor]
\label{lem:heap_search_complexity}
Let $m$ be the maximal arity of a function symbol and $|R|$ the number of derivation rules.
After $t$ programs have been generated from $S$ by the main loop, the next call $\ComputeSuccessor(\cdot, S)$ of the main loop performs at most $|\Gamma| (m + 1)$ heap operations, each on a heap of size at most $|R| + m(t+2)$.
In particular, it runs in time $O(\log t)$.
\end{lemma}

\begin{proof}
We say a call to $\ComputeSuccessor(\program'', Y)$ is \emph{new} if it branches to the \textbf{else} case (i.e.\ $\succ_Y(\program'')$ is not yet defined), and \emph{cached} otherwise.
A cached call takes constant time, and a new call to $\ComputeSuccessor(\cdot, Y)$ performs at most $m$ insertions into $\heap_Y$ (one per argument) and one pop from $\heap_Y$.

\medskip
\noindent\textbf{At most one new call per non-terminal.}
Consider an outer call $\ComputeSuccessor(\program, S)$ of the main loop.
All the calls it triggers are on sub-programs of $\program$, which were all generated before the outer call started, by (I2).
By (I1), a call $\ComputeSuccessor(\program'', Y)$ is new only if $\program''$ is the last program generated from $Y$.
After a new call on $Y$, the last program generated from $Y$ is the program it has just popped, which is not a sub-program of $\program$ since it was generated during the outer call.
Hence during the outer call there is at most one new call on each non-terminal $Y$, and the outer call performs at most $|\Gamma| (m+1)$ heap operations.

\medskip
\noindent\textbf{Size of the heaps.}
Initially $\heap_Y$ contains at most $|R|$ programs.
The initialisation and each call of the main loop make at most one new call on $Y$, and each new call inserts at most $m$ programs into $\heap_Y$.
During the call that generates the $(t+1)$-th program from $S$, at most $t + 2$ such calls have been made, so $\heap_Y$ contains at most $|R| + m(t+2)$ programs.

Each heap operation costs $O(\log(|R| + m(t+2)))$, so the call runs in time $O(|\Gamma| (m+1) \log(|R| + m(t+2))) = O(\log t)$ for a fixed grammar.
\end{proof}

The proof of Lemma~\ref{lem:heap_search_correctness} shows that the next program of $\program$ falls in one of two categories:
\begin{itemize}
	\item it is already in $\heap_X$,
	\item it is obtained from $\program$ by replacing one of its arguments by its next program (for the corresponding non-terminal).
\end{itemize}
This is why the \textbf{for} loop of Algorithm~\ref{algo:heap_search_next} expands $\program$ before popping $\heap_X$.

\paragraph*{Comparison with the original presentation.}
In the original presentation of \heapS~\cite[Appendix C]{Fijalkow.Lagarde.ea:2022}, the order is different: $\ComputeSuccessor(\program, X)$ first pops $\program'$ from $\heap_X$, and then expands $\program'$ instead of $\program$.
Both versions insert the same programs, but at different times.
The original order is correct for acyclic grammars, where no non-terminal can be reached from itself, but not in general.
Consider the rules $W \rightarrow k(U)$, $U \rightarrow u_1$ and $U \rightarrow q(W)$, of costs $1$, $4$ and $0.5$.
The minimal programs are $u_1$ and $k(u_1)$.
With the original order, the initial call $\ComputeSuccessor(\bot, W)$ pops $k(u_1)$ and expands it, which calls $\ComputeSuccessor(u_1, U)$.
This call pops $q(k(u_1))$ and expands it, which calls $\ComputeSuccessor(k(u_1), W)$ while $\heap_W$ is empty: the next program $k(q(k(u_1)))$ of $k(u_1)$ would only be inserted at the end of the expansion of $k(u_1)$.
In Algorithm~\ref{algo:heap_search_next}, a call only expands its own argument, and all recursive calls are on strict sub-programs of this argument, so this situation cannot occur.
Our running example is not acyclic, because of the rule $\ntstr \rightarrow \text{concat}(\ntstr,\ntstr)$, and neither are the grammars of our experiments.
The implementation of \heapS used in Section~\ref{sec:experiments} follows Algorithm~\ref{algo:heap_search_next}.

\begin{algorithm}[H]
  \caption{Heap Search: initialisation\label{algo:heap_search_init}}
  \begin{algorithmic}[1]
		\State compute $\minprogram(X)$ a program of minimal cost from $X$ for each non-terminal $X$
		\State $\seen$: set of programs
    \For{$X$ non-terminal}
			\State $\heap_X$: heap of programs
			\State $\succ_X$: mapping from programs to programs
    \EndFor
    \Statex

    \Function{insert}{$\program, X$}:
		\State add $\program$ to $\heap_X$ with value $\cost(\program)$
		\State add $\program$ to $\seen$
		\EndFunction
    \Statex

		\For{$r : X \to f(X_1,\dots,X_k)$ derivation rule}\label{line:hs_init_for}
			\Let{$\program$}{$f(\minprogram(X_1), \dots, \minprogram(X_k))$}
			\State \myinsert($\program, X$)
		\EndFor\label{line:hs_init_endfor}
    \Statex

		\For{$X$ non-terminal}
			\State $\ComputeSuccessor(\bot, X)$
	  	\Comment $\bot$ is a dummy program
		\EndFor
  \end{algorithmic}
\end{algorithm}

\begin{algorithm}[H]
  \caption{Heap Search: main loop\label{algo:heap_search_next}}
  \begin{algorithmic}[1]
  	\Let{$\program$}{$\bot$}
  	\Comment $\bot$ is a dummy program
  	\While{\text{True}}
  		\Let{$\program$}{$\ComputeSuccessor(\program, S)$}
  		\Comment $S$ is the initial non-terminal
  		\State \textbf{yield} $\program$\label{line:hs_yield}
  	\EndWhile
  	\Statex

    \Function{ComputeNext}{$\program, X$}:
    \Comment returns the next program of $\program$ from $X$
    	\If{$\succ_X(\program)$ has already been computed}
    		\State \textbf{return} $\succ_X(\program)$
    		\Comment memoization
    	\Else
			\If{$\program \neq \bot$}
				\State write $\program$ as $f(\program_1,\dots,\program_k)$\label{line:hs_write}
				\Comment $X \rightarrow f(X_1,\dots,X_k)$ is the unique rule that produced $\program$
				\For{$i$ from $1$ to $k$}\label{line:hs_for}
					\Comment expands $\program$
					\Let{$\program'_i$}{$\ComputeSuccessor(\program_i, X_i)$}
					\Let{$\program''_i$}{$f(\program_1,\dots,\program'_i,\dots,\program_k)$}
					\If{$\program''_i$ not in $\seen$}
						\State \myinsert($\program''_i, X$)
					\EndIf
				\EndFor\label{line:hs_endfor}
			\EndIf
			\Let{$\program'$}{$\pop(\heap_X)$}
			\Let{$\succ_X(\program)$}{$\program'$}
			\State \textbf{return} $\program'$
		\EndIf
	\EndFunction
  \end{algorithmic}
\end{algorithm}

Line~\ref{line:hs_write} of Algorithm~\ref{algo:heap_search_next} is not an assignment: it names the arguments of $\program$.
Since the grammar is deterministic, $\program$ was produced by a unique derivation rule $X \rightarrow f(X_1,\dots,X_k)$, which tells us which non-terminal $X_i$ to use for each argument $\program_i$ in the recursive calls.
If $\program$ was produced by a terminal rule then $k = 0$ and the \textbf{for} loop does nothing.

\paragraph*{An example by hand.}
We consider the grammar and associated costs defined in Figure~\ref{fig:running_example_grammar}.
In a single iteration, Algorithm~\ref{algo:minimal_programs} finds $\minprogram(\ntstr) = H, \mincost(\ntstr) = 1.1$ and $\minprogram(\ntint) = \text{var}, \mincost(\ntint) = 1.8$.
During initialisation, we insert the following programs:
$H, W, \text{cast}(\text{var}), \text{concat}(H,H)$ into $\heap_\ntstr$ and $\text{var}, 1, \text{add}(\text{var},\text{var})$ into $\heap_\ntint$.
We then run $\ComputeSuccessor(\bot,\ntstr)$, which pops $H$, and $\ComputeSuccessor(\bot,\ntint)$, which pops $\text{var}$.
At this point, the data structures are as follows, with costs indicated below programs:
\[
\succ_\ntstr(\bot) = H,\ \succ_\ntint(\bot) = \text{var}
\]
\[
\begin{array}{lll}
\heap_\ntstr & = & \set{\underbrace{W}_{2.0}, \underbrace{\text{cast}(\text{var})}_{6.2}, \underbrace{\text{concat}(H,H)}_{7.5}} \\
\heap_\ntint & = & \set{\underbrace{1}_{3.3}, \underbrace{\text{add}(\text{var},\text{var})}_{8.9}} \\
\seen   & = & \set{H, W, \text{concat}(H,H), \text{cast}(\text{var}), \text{var}, 1, \text{add}(\text{var},\text{var})}
\end{array}
\]
The main loop only calls $\ComputeSuccessor(\cdot,\ntstr)$; calls for $\ntint$ only happen as recursive calls.
Let us analyse the first six calls of the main loop:
\begin{enumerate}
	\item $\ComputeSuccessor(\bot,\ntstr)$ returns $H$, already computed during initialisation.
	\item $\ComputeSuccessor(H,\ntstr)$: the next program of $H$ is not known yet.
	We first expand $H$, which does nothing since $H$ has no arguments.
	We then pop $W$ from $\heap_\ntstr$, set $\succ_\ntstr(H) = W$, and return $W$.
	Indeed $W$ is the cheapest program from $\ntstr$ of cost larger than $\cost(H)$.
	Note that only $W$ is returned: the other programs in $\heap_\ntstr$ stay there until they are the next program of some program.
	\item $\ComputeSuccessor(W,\ntstr)$: as in the previous step, expanding $W$ does nothing, and we pop $\text{cast}(\text{var})$ from $\heap_\ntstr$, set $\succ_\ntstr(W) = \text{cast}(\text{var})$, and return it.
	\item $\ComputeSuccessor(\text{cast}(\text{var}),\ntstr)$: we first expand $\text{cast}(\text{var})$ (lines~\ref{line:hs_for} to~\ref{line:hs_endfor}).
	Its only argument is $\text{var}$, from $\ntint$, so we run $\ComputeSuccessor(\text{var},\ntint)$, which pops $1$ from $\heap_\ntint$, sets $\succ_\ntint(\text{var}) = 1$, and returns $1$.
	We consider $\text{cast}(1)$, currently not in $\seen$, so it is inserted.
	We then pop $\text{concat}(H,H)$ from $\heap_\ntstr$ and set $\succ_\ntstr(\text{cast}(\text{var})) = \text{concat}(H,H)$.
	After this call the heaps are as follows:
\[
\begin{array}{lll}
\heap_\ntstr & = & \set{\underbrace{\text{cast}(1)}_{7.7}} \\
\heap_\ntint & = & \set{\underbrace{\text{add}(\text{var},\text{var})}_{8.9}} \\
\end{array}
\]
	\item $\ComputeSuccessor(\text{concat}(H,H),\ntstr)$: we expand $\text{concat}(H,H)$.
	Both arguments are $H$, from $\ntstr$, and $\ComputeSuccessor(H,\ntstr)$ returns $W$, already computed in step 2.
	We insert $\text{concat}(W,H)$ and $\text{concat}(H,W)$, both of cost $8.4$, into $\heap_\ntstr$.
	We then pop $\text{cast}(1)$ and return it.
	\item $\ComputeSuccessor(\text{cast}(1),\ntstr)$: we expand $\text{cast}(1)$.
	Its argument is $1$, from $\ntint$, and the next program of $1$ is not known yet, so $\ComputeSuccessor(1,\ntint)$ expands $1$, which does nothing, and pops $\text{add}(\text{var},\text{var})$ from $\heap_\ntint$.
	Back to the first call, we insert $\text{cast}(\text{add}(\text{var},\text{var}))$, of cost $13.3$, into $\heap_\ntstr$, then pop one of the two programs of cost $8.4$, say $\text{concat}(W,H)$, and return it.
\end{enumerate}

\subsubsection{Observational equivalence in \heapS}
\label{sec:heap_search_oe}
\citet[Section 4.3.1]{Ameen.Lelis:2023} observe that \heapS becomes incomplete if a program that is observationally equivalent to an earlier one is simply discarded.
In their example, with the rules $I \rightarrow 1 \mid 2 \mid I + I$, the program $1 + 1$ is equivalent to $2$; if it is thrown away when popped, its children $2 + 1$ and $1 + 2$ are never inserted, and they cannot be inserted through another program.
This is correct, and it means that observational equivalence must be implemented with care in \heapS.
Our implementation of \heapS, written by some of the authors of~\cite{Fijalkow.Lagarde.ea:2022} in the DeepSynth framework, does the following: when a program $\program'$ is popped, it is checked for equivalence.
If it is equivalent to a program already generated, it is expanded right away, and removed from the sequence of next programs: if $\succ_X(\program) = \program'$ and $\succ_X(\program') = \program'''$, we set $\succ_X(\program) = \program'''$, so that $\program'$ is never used again as an argument.
In the example, $2 + 1$ and $1 + 2$ are inserted when $1 + 1$ is popped, and then $1 + 1$ is removed.
Since a removed program is expanded before being removed, removing it from the sequence of next programs loses no program, and the algorithm remains complete up to observational equivalence.
The price to pay is that \heapS constructs and expands the redundant programs anyway, so it benefits much less from observational equivalence than \beeS and \ecoS, where a redundant program is discarded right after being constructed, before it is ever combined with other programs.
This is consistent with the observation of~\citet{Ameen.Lelis:2023} that \heapS performs poorly on benchmarks where observational equivalence matters.

\subsubsection{Limitations of \heapS}
There are two limitations of \heapS:
\begin{itemize}
	\item The first is the structure of recursive calls when updating the data structures, which insert a lot of programs.
More precisely, the issue is that these programs are added to the data structures although they are not generated yet, because they may have much larger costs. In other words, when generating a program of cost $c$, \heapS needs to consider many programs that have costs potentially much larger than $c$.
In step 6 of the example above, generating $\text{concat}(W,H)$ of cost $8.4$ requires generating $\text{add}(\text{var},\text{var})$ of cost $8.9$ and inserting $\text{cast}(\text{add}(\text{var},\text{var}))$ of cost $13.3$.
Later, expanding $\text{add}(\text{var},\text{var})$ inserts $\text{add}(1,\text{var})$ and $\text{add}(\text{var},1)$, of cost $10.4$.
This makes the algorithm very memory hungry.
	\item The second is that it needs to explicitly build all the programs it considers, again very heavy on memory consumption.
\end{itemize}

\subsection{\beeS}
\label{sec:bee_search}

\beeS was introduced by~\citet{Ameen.Lelis:2023}.
It maintains three objects:
\begin{itemize}
	\item $\generated$: stores the set of programs generated so far, organised by costs.
	Concretely, it is a mapping from costs to sets of programs: $\generated[c]$ is the set of generated programs of cost $c$.
	\item $\indextocost$: a list of the costs of the generated programs.
	Let us write $\indextocost = [c_1,\dots,c_\ell]$, then $c_1 < \dots < c_\ell$ and $\generated[c_i]$ is defined.
	\item $\queue$: stores information about which programs to generate next.
	Concretely, it is a priority queue of \textit{cost tuples} ordered by costs, that we define now.
\end{itemize}

A \textit{cost tuple} is a pair consisting of a derivation rule $r : X \rightarrow f(X_1,\dots,X_k)$ and a tuple $n = (n_1,\dots,n_k) \in \mathbb{N}^k$.
For derivation rules $r : X \rightarrow a$, cost tuples are of the form $(r, \emptyset)$.
A cost tuple represents a set of programs:
$(r, n)$ represents all programs generated by the rule $r$ where the $i$\textsuperscript{th} argument is any program in $\generated[\indextocost[n_i]]$.
The cost of a cost tuple $t = (r, n)$ is defined as
\[
\cost(t) = \cost(r) + \sum_{i = 1}^k \indextocost[n_i].
\]
All programs represented by a cost tuple have the same cost, namely the cost of the tuple.
In the terminology of heuristic search, cost tuples are the nodes of the search, and the children of $(r,n)$ are the $k$ cost tuples obtained by increasing one index of $n$ by $1$.

\vskip1em
The initialisation is given in Algorithm~\ref{algo:bee_search_init}.
The main function is \textsc{Output} in Algorithm~\ref{algo:bee_search_main}, which is called repeatedly and indefinitely.
A single call to \textsc{Output} generates \textbf{all} programs represented by the cost tuple $t$ found by popping $\queue$. There are two cases:
\begin{itemize}
	\item Line $5$ if $t = (r : X \rightarrow a, \emptyset)$.
	The cost of $t$ is $\cost(r)$ and the single program generated is $a$.
	To update the data structure, we check whether $c \neq \indextocost[-1]$, meaning that the last generated program had cost strictly less than $c$. In that case we assign a new empty set to $\generated[c]$, otherwise $\generated[c]$ already exists, and in both cases we add $a$ to $\generated[c]$.
	\item Line $12$ if $t = (r : X \rightarrow f(X_1,\dots, X_k), n : (n_1,\dots,n_k))$.
	The cost of $t$ is computed in line 13.
	Lines 14 to 16 assign a new empty set to $\generated[c]$ if it did not exist already.
	Lines 17 to 23 generate programs with cost $c$.
	Lines 24 to 29 update the data structures by adding the necessary cost tuples.
\end{itemize}
Observational equivalence is used in lines 20 and 21: a program that is equivalent to a program already generated is discarded.

Note that \beeS never needs to know costs that have not been reached yet.
Indeed, when the children of $t = (r, n)$ are pushed in lines 24 to 29, the value $\indextocost[n_i + 1]$ is always defined, because $\indextocost[n_i] < c$ and $c$ itself was added to $\indextocost$ in line 15.
This will not be true anymore when using one list of costs per non-terminal, and it is the reason for the frugal expansion of \ecoS (Section~\ref{sec:frugal_expansion}).

\begin{algorithm}
  \caption{Bee Search: initialisation\label{algo:bee_search_init}}
  \begin{algorithmic}[1]
	\State $\generated$: mapping from costs to sets of programs
	\State $\indextocost$: list of costs
	\State $\queue$: priority queue of cost tuples ordered by costs
   	\Statex
	\Let{$c$}{minimal cost of a program}
	\State add $c$ to $\indextocost$
   	\Statex
	\For{$r : X \to f(X_1,\dots,X_k)$ derivation rule}
		\Let{$t$}{$(r, \underbrace{(0,\dots,0)}_{k \text{ times}})$}
		\Let{$\cost(t)$}{$\cost(r) + k \times c$}
		\State add $t$ to $\queue$ with value $\cost(t)$
	\EndFor
  \end{algorithmic}
\end{algorithm}

\begin{algorithm}
  \caption{Bee Search: main loop\label{algo:bee_search_main}}
  \begin{algorithmic}[1]
  	\While{\text{True}}
  		\State $\textsc{Output}()$
  	\EndWhile
  	\Statex
    \Function{output()}{}:
	\Comment generates all programs represented by the cost tuple of minimal cost
		\Let{$t$}{$\pop(\queue)$}

		\If{$t = (r : X \rightarrow a, \emptyset)$}
			\Let{$c$}{$\cost(r)$}
			\Comment computes $\cost(t)$
			\If{$c \neq \indextocost[-1]$}
				\State add $c$ to $\indextocost$
				\Comment the last generated program did not have cost $c$
				\Let{$\generated[c]$}{$\emptyset$}
			\EndIf
			\State add $a$ to $\generated[c]$
			\State \textbf{yield} $a$

	  	\Statex
		\Else\ $t = (r : X \rightarrow f(X_1,\dots, X_k), n : (n_1,\dots,n_k))$
			\Let{$c$}{$\cost(r) + \sum_{i = 1}^k \indextocost[n_i]$}
			\Comment computes $\cost(t)$
			\If{$c \neq \indextocost[-1]$}
				\State add $c$ to $\indextocost$
				\Comment the last generated program did not have cost $c$
				\Let{$\generated[c]$}{$\emptyset$}
			\EndIf

			\For{$\program_1,\dots,\program_k$ in $\bigotimes_{i = 1}^k \generated[\indextocost[n_i]]$}
			\Comment{generates programs}
				\If{for all $i \in \set{1,\dots,k}$, $\program_i$ is generated from $X_i$}
					\Let{$\program$}{$f(\program_1,\dots,\program_k)$}
					\If{$\program$ is observationally equivalent to a generated program}
						\State \textbf{continue}
						\Comment only if observational equivalence is used
					\EndIf
					\State add $\program$ to $\generated[c]$
					\State \textbf{yield} $\program$
				\EndIf
			\EndFor

	        \For{$i$ from $1$ to $k$}
			\Comment{updates the data structure}
				\Let{$n'$}{$n$}
				\Let{$n'_i$}{$n_i + 1$}
				\If{$t' = (r,n')$ not in $\queue$}
					\Let{$\cost(t')$}{$c + \indextocost[n'_i] - \indextocost[n_i]$}
					\Comment efficient computation of $\cost(t')$
					\State add $t'$ to $\queue$ with value $\cost(t')$
				\EndIf
			\EndFor
		\EndIf
	\EndFunction
  \end{algorithmic}
\end{algorithm}

\paragraph*{An example by hand.}
We consider the grammar and associated costs defined in Figure~\ref{fig:running_example_grammar}.
The minimal cost of a program is $1.1$, so we set $\indextocost = [1.1]$.
During initialisation, we add the following cost tuples:
\[
(r_1, \emptyset), (r_2, \emptyset), (r_3, (0)), (r_4, (0,0)), (r_5, \emptyset), (r_6, \emptyset), (r_7, (0,0)).
\]
At this point, the queue is as follows, with costs indicated below cost tuples:
\[
\queue = \set{\underbrace{(r_1, \emptyset)}_{1.1},
\underbrace{(r_5, \emptyset)}_{1.8},
\underbrace{(r_2, \emptyset)}_{2.0},
\underbrace{(r_6, \emptyset)}_{3.3},
\underbrace{(r_3, (0))}_{5.5},
\underbrace{(r_4, (0,0))}_{7.5},
\underbrace{(r_7, (0,0))}_{7.5}}
\]
Let us analyse the first calls to $\textsc{Output}$:
\begin{enumerate}
	\item We pop $(r_1, \emptyset)$ and add $H$ to $\generated[1.1]$.
	\item We pop $(r_5, \emptyset)$, add $1.8$ to $\indextocost$ and $\text{var}$ to $\generated[1.8]$.
	\item We pop $(r_2, \emptyset)$, add $2.0$ to $\indextocost$ and $W$ to $\generated[2.0]$.
	\item We pop $(r_6, \emptyset)$, add $3.3$ to $\indextocost$ and $1$ to $\generated[3.3]$.
	At this point we have $\indextocost = [1.1,\ 1.8,\ 2.0,\ 3.3]$.
	\item We pop $(r_3, (0))$, of cost $\cost(r_3) + \indextocost[0] = 4.4 + 1.1 = 5.5$, and add $5.5$ to $\indextocost$.
	We try generating programs: $\generated[\indextocost[0]] = \set{H}$.
	Since $H$ is not generated from $\ntint$, the rule $r_3$ does not apply ($\text{cast}(H)$ is not a program of our DSL), and the algorithm does not generate programs at this step.
	We then update the data structure, adding $(r_3, (1))$ to $\queue$ with cost $5.5 + 1.8 - 1.1 = 6.2$.
At this point, the queue is as follows, with costs indicated below cost tuples:
\[
\queue = \set{\underbrace{(r_3, (1))}_{6.2},
\underbrace{(r_4, (0,0))}_{7.5},
\underbrace{(r_7, (0,0))}_{7.5}}
\]
	\item We pop $(r_3, (1))$, of cost $\cost(r_3) + \indextocost[1] = 4.4 + 1.8 = 6.2$.
	We try generating programs: $\generated[\indextocost[1]] = \set{\text{var}}$.
	The program $\text{var}$ is generated from $\ntint$, so the algorithm generates $\text{cast}(\text{var})$.
	We then update the data structure, adding $(r_3, (2))$ to $\queue$ with cost $6.2 + 2.0 - 1.8 = 6.4$.
At this point, the queue is as follows, with costs indicated below cost tuples:
\[
\queue = \set{\underbrace{(r_3, (2))}_{6.4}, \underbrace{(r_4, (0,0))}_{7.5}, \underbrace{(r_7, (0,0))}_{7.5}}
\]
\end{enumerate}

\subsubsection{Limitations of \beeS}
The main limitation of \beeS is that there may be calls to \textsc{Output} where the algorithm does not generate any program, as in the fifth iteration in our example.
The next call has the same problem: the cost tuple $(r_3, (2))$ refers to $\generated[2.0] = \set{W}$, and $\text{cast}(W)$ is not a program of our DSL.
The reason is that $\generated$ contains programs from all non-terminals, so a cost tuple may represent no program at all.
The number of such empty cost tuples popped between two generated programs is not bounded by a constant in general.
For instance, consider a rule $S \rightarrow g(I,I)$, and assume that the programs from $I$ all have even costs while programs of all integer costs exist.
For an odd value $c$, all cost tuples for this rule of cost $\cost(g) + c$ are empty, and there are about $c$ of them, which may be popped one after the other.
 
\section{\ecoS}
\label{sec:enumeration_algorithms}
\label{sec:eco_search}
\ecoS combines ideas from its two predecessors and adds new ones, and we first list which ideas are inherited and which are new.
\begin{itemize}
	\item The \emph{cost tuple representation} comes from \beeS~\cite{Ameen.Lelis:2023} (Section~\ref{sec:bee_search}).
	\item The \emph{per non-terminal data structures} come from \heapS~\cite{Fijalkow.Lagarde.ea:2022} (Section~\ref{sec:heap_search}).
	\item \emph{Frugal expansion} is new, and it is what makes the two previous ideas work together: with one list of costs per non-terminal, a cost tuple may refer to costs of programs that have not been generated yet. Frugal expansion reads these costs off the queues without generating the corresponding programs, which avoids the cascade of recursive calls of \heapS (Section~\ref{sec:frugal_expansion}).
	\item \emph{Bucket queues} are a classical data structure from shortest path algorithms~\cite{Dial:1969,Cherkassky.Goldberg.ea:1999}. What is new is the proof that they can be used here: we show that at any time all costs in a queue lie in a window of constant width (Lemmas~\ref{lemma:cost_difference} and~\ref{lemma:bound_cost}). Together with the previous ingredients, this gives constant delay (Theorem~\ref{thm:constant_delay}).
\end{itemize}
Table~\ref{tab:algorithm_comparison} summarises which ingredients each algorithm uses.

\begin{table}[h]
\centering
\small
\caption{Comparison of best-first bottom-up synthesis algorithms.
``Per-NT structures'' indicates whether separate data structures are maintained per non-terminal.
``Cost tuples'' indicates use of the cost-tuple representation.
``Frugal expansion'' indicates that the cost of the next programs of a non-terminal is read from its queue instead of being computed by generating these programs.
``Bucket queue'' indicates $O(1)$ queue operations enabled by Lemma~\ref{lemma:bound_cost}.}
\label{tab:algorithm_comparison}
\begin{tabular}{lccccc}
\toprule
Algorithm & Per-NT structures & Cost tuples & Frugal expansion & Bucket queue & Delay \\
\midrule
\heapS~\cite{Fijalkow.Lagarde.ea:2022}               & \checkmark & $\times$   & $\times$   & $\times$ & $O(\log t)$ \\
\beeS~\cite{Ameen.Lelis:2023}                         & $\times$   & \checkmark & n/a   & $\times$ & $O(\log t)$ per tuple \\
\ecoS without bucketing (Def.~\ref{def:eco_no_bucket}) & \checkmark & \checkmark & \checkmark & $\times$ & $O(\log t)$ \\
\ecoS (this work)                                      & \checkmark & \checkmark & \checkmark & \checkmark & $O(1)$ \\
\bottomrule
\end{tabular}
\end{table}

We describe the data structures in Section~\ref{sec:eco_data_structures}, the algorithm in Section~\ref{sec:eco_algorithm}, frugal expansion in Section~\ref{sec:frugal_expansion}, and an example in Section~\ref{sec:eco_example}.
We prove correctness in Section~\ref{sec:eco_correctness}, and constant delay in Section~\ref{sec:bucketing}.
Section~\ref{sec:eco_ablation} discusses what happens when only some of the ingredients are used, and Section~\ref{sec:eco_extensions} discusses observational equivalence and post-generation cost functions.

\subsection{Data structures}
\label{sec:eco_data_structures}

\ecoS uses the same three objects as \beeS (Section~\ref{sec:bee_search}), but with one copy for each non-terminal~$X$:
$\generated_X$ maps costs to sets of programs generated from $X$,
$\indextocost_X$ is the increasing list of costs of programs from $X$,
and $\queue_X$ is a priority queue of cost tuples for the rules from $X$.
In addition, $\done_X$ is the number of indices $\ell$ for which all programs from $X$ of cost $\indextocost_X[\ell]$ have been generated.
Cost tuples are defined as in \beeS, except that indices refer to the list of the non-terminal of the corresponding argument: a cost tuple $(r : X \rightarrow f(X_1,\dots,X_k), (n_1,\dots,n_k))$ represents all programs $f(\program_1,\dots,\program_k)$ where each $\program_i$ is in $\generated_{X_i}[\indextocost_{X_i}[n_i]]$, and its cost is
\[
\cost(r) + \sum_{i = 1}^k \indextocost_{X_i}[n_i].
\]
Because each $\program_i$ is taken from $\generated_{X_i}$, every combination is a program of the grammar, so the type check of \beeS (line 18 of Algorithm~\ref{algo:bee_search_main}) disappears, and a cost tuple never represents an empty set of programs (we prove this in Lemma~\ref{lem:eco_invariants}).
This solves the limitation of \beeS discussed in Section~\ref{sec:bee_search}.

For a non-terminal $X$, we write $c_X^0 < c_X^1 < c_X^2 < \cdots$ for the distinct costs of programs generated from $X$ (the sequence is finite if $X$ generates finitely many programs).
We will show that $\indextocost_X[\ell] = c_X^\ell$ whenever it is defined.

\subsection{The algorithm}
\label{sec:eco_algorithm}

The pseudocode is given in Algorithms~\ref{algo:beap_search_init} and~\ref{algo:beap_search_main_loop}.
The initialisation (Algorithm~\ref{algo:beap_search_init}) uses Algorithm~\ref{algo:minimal_programs} to compute $\mincost(X)$ for each non-terminal $X$, and pushes for each rule the cost tuple whose indices are all $0$.
It then sets $\indextocost_X[0]$ to the cost of the cheapest tuple in $\queue_X$, which is $\mincost(X)$.

The main function is $\textsc{Output}(X, \ell)$ (Algorithm~\ref{algo:beap_search_main_loop}).
It generates all programs from $X$ of cost $\indextocost_X[\ell]$, and is called only when $\ell = \done_X$, that is, when the programs of the previous costs have already been generated.
It is informally decomposed into a \emph{generation} part (lines 10 to 20), which is in charge of generating programs,
and an \emph{update} part (lines 21 to 26), which pushes the children of the cost tuple.
The main loop calls $\textsc{Output}(S, 0)$, $\textsc{Output}(S, 1)$, and so on, and yields the programs from $S$.
The calls for other non-terminals are recursive calls (line 14): to generate the programs represented by a cost tuple $(r, (n_1,\dots,n_k))$, we need the programs from $X_i$ of cost $\indextocost_{X_i}[n_i]$, so if they have not been generated yet ($n_i = \done_{X_i}$), we generate them first.
The programs generated by recursive calls are stored in the corresponding $\generated_{X_i}$, but they are not passed to the main loop.

The generation part has the same structure as in \beeS (lines 17 to 23 of Algorithm~\ref{algo:bee_search_main}), with two differences: the arguments come from per non-terminal sets, so there is no type check, and the sets are obtained by recursive calls when needed.
The update part pushes the same children as \beeS (lines 24 to 29 of Algorithm~\ref{algo:bee_search_main}), with the difference that $\indextocost_{X_i}$ depends on the non-terminal.
Lines 28 and 29 of Algorithm~\ref{algo:beap_search_main_loop} implement frugal expansion, which we explain now.

\begin{algorithm}
\caption{\ecoS: initialisation\label{algo:beap_search_init}}
\begin{algorithmic}[1]
\State compute $\mincost(X)$, the minimal cost of a program from $X$, for each non-terminal $X$
\Comment Algorithm~\ref{algo:minimal_programs}
\For{$X$ non-terminal}
	\State $\generated_X$: mapping from costs to sets of programs
	\State $\indextocost_X$: list of costs
	\State $\queue_X$: priority queue of cost tuples ordered by costs
	\Let{$\done_X$}{$0$}
\EndFor
\Statex
\For{$r : X \to f(X_1,\dots,X_k)$ derivation rule}
	\Let{$c$}{$\cost(r) + \sum_{i = 1}^k \mincost(X_i)$}
	\Comment the cost of the (implicit) program $f(\minprogram(X_1), \dots, \minprogram(X_k))$
	\State add $(r, (0,\dots,0))$ to $\queue_X$ with value $c$
\EndFor
\For{$X$ non-terminal}
	\Let{$\indextocost_X[0]$}{$\cost(\peek(\queue_X))$}
	\Comment equal to $\mincost(X)$
\EndFor
\end{algorithmic}
\end{algorithm}

\begin{algorithm}
\caption{\ecoS: main loop}\label{algo:beap_search_main_loop}
\begin{algorithmic}[1]
\State $\ell \gets 0$
\While{\text{True}}
	\For{$\program$ in $\textsc{Output}(S, \ell)$}
		\State \textbf{yield} $\program$
	\EndFor
    \Let{$\ell$}{$\ell + 1$}
\EndWhile
\Statex

\Function{output}{$X, \ell$}:
	\Comment generates all programs from $X$ of cost $\indextocost_X[\ell]$; requires $\ell = \done_X$
	\Let{$c$}{$\indextocost_X[\ell]$}
	\Let{$\generated_X[c]$}{$\emptyset$}
    \While{$\queue_X$ is not empty and $\cost(\peek(\queue_X)) = c$}
		\Comment iterate over the cost tuples of cost $c$
		\Let{$t$}{$\pop(\queue_X)$}
		\State $t = (r : X \rightarrow f(X_1,\dots, X_k), n : (n_1,\dots,n_k))$
		\Comment $k = 0$ for a terminal rule
		\For{$i$ from $1$ to $k$}
			\If{$n_i = \done_{X_i}$}
				\State run $\textsc{Output}(X_i, n_i)$
				\Comment recursive call: generates the arguments
			\EndIf
		\EndFor
   	    \For{$\program_1,\dots,\program_k$ in $\bigotimes_{i = 1}^k \generated_{X_i}[\indextocost_{X_i}[n_i]]$}
			\Comment generates programs
    		\Let{$\program$}{$f(\program_1,\dots,\program_k)$}
			\If{$\program$ is observationally equivalent to a program in $\generated_X$}
				\State \textbf{continue}
				\Comment only if observational equivalence is used
			\EndIf
   			\State add $\program$ to $\generated_X[c]$
           	\State \textbf{yield} $\program$
		\EndFor
	    \For{$i$ from $1$ to $k$}
			\Comment updates the data structure
       		\Let{$n'$}{$n$}
       		\Let{$n'_i$}{$n_i + 1$}
			\If{$\indextocost_{X_i}[n'_i]$ is defined and $t' = (r, n')$ is not in $\queue_X$}
				\Let{$\cost(t')$}{$\cost(t) + \indextocost_{X_i}[n'_i] - \indextocost_{X_i}[n_i]$}
				\State add $t'$ to $\queue_X$ with value $\cost(t')$
			\EndIf
	  	\EndFor
	\EndWhile
	\Let{$\done_X$}{$\ell + 1$}
	\If{$\queue_X$ is not empty}
		\Let{$\indextocost_X[\ell + 1]$}{$\cost(\peek(\queue_X))$}
		\Comment frugal expansion: the next cost is known, the programs are not generated
	\EndIf
\EndFunction
\end{algorithmic}
\end{algorithm}

To make our pseudocode as readable as possible we use the generator syntax of Python, as in Section~\ref{sec:predecessor_algorithms}.
Line 14 runs the generator $\textsc{Output}(X_i, n_i)$ until it ends.
Line 11 is not an assignment: it names the components of the cost tuple $t$.
Lines 17 and 18 are only used when observational equivalence is enabled; we first analyse the algorithm without them, and come back to them in Section~\ref{sec:eco_extensions}.
The test ``$t'$ is not in $\queue_X$'' in line 24 avoids pushing the same cost tuple twice (a cost tuple has up to $k$ parents); we discuss its implementation in the proof of Theorem~\ref{thm:constant_delay}.
The test ``$\indextocost_{X_i}[n'_i]$ is defined'' only fails when $X_i$ generates finitely many programs and all of them have been generated.

\subsection{Frugal expansion}
\label{sec:frugal_expansion}

When \ecoS pushes the child $t' = (r, n')$ of a cost tuple $t = (r,n)$, with $n'_i = n_i + 1$, it needs the cost of $t'$, hence the value $\indextocost_{X_i}[n_i + 1]$, that is, the cost of the programs from $X_i$ that come right after the ones of cost $\indextocost_{X_i}[n_i]$.
In \beeS this value is always known (Section~\ref{sec:bee_search}).
With per non-terminal lists it may not be: the programs from $X_i$ of that cost may not have been generated yet.
There are two ways to obtain this value.
\begin{itemize}
	\item The eager way, which is what \heapS does: generate the next programs from $X_i$ right away, that is, call $\textsc{Output}(X_i, n_i + 1)$.
	This call may itself trigger recursive calls for the arguments of the programs from $X_i$, and so on. As illustrated in step 6 of the example of Section~\ref{sec:heap_search}, this generates programs whose costs are larger than the cost $c$ of the programs currently generated, and it is the main source of inefficiency of \heapS.
	\item The frugal way, which is what \ecoS does: observe that the cost we need is the cost of the cheapest cost tuple in $\queue_{X_i}$, which can be read with a $\peek$ operation without popping or generating anything.
	The programs from $X_i$ of this cost are only generated when a cost tuple that uses them is popped, that is, when they are actually needed to generate programs of the current cost.
\end{itemize}
In Algorithm~\ref{algo:beap_search_main_loop} the peek is performed at the end of $\textsc{Output}(X, \ell)$ (line 29): as soon as the programs of cost $\indextocost_X[\ell]$ have been generated, we record the cost of the next ones.
We prove in Lemma~\ref{lem:eco_invariants} that this value is correct, that is, that no cost tuple of smaller cost is pushed into $\queue_X$ later.
As a consequence, $\indextocost_{X_i}[n_i + 1]$ is always defined in line 24, unless $X_i$ is exhausted.

Frugal expansion should not be confused with memoization.
Memoization is the fact that each call $\textsc{Output}(X, \ell)$ is executed once, and later requests read the result from $\generated_X$ (line 13 checks $n_i = \done_{X_i}$); it is already present in \heapS through the mapping $\succ_X$.
Frugal expansion is about \emph{when} the programs are generated: \heapS (and \ecoS with eager expansion) generates the next programs of a non-terminal when it only needs their cost, while \ecoS only computes their cost.

As a consequence, recursive calls only happen in the generation part (line 14), and never in the update part.
Moreover, a recursive call $\textsc{Output}(Y, j)$ only happens when $j = \done_Y$, and costs strictly decrease along the call stack (the cost of a program is larger than the cost of its arguments).
Hence for each non-terminal $Y$, at most one call $\textsc{Output}(Y, \cdot)$ is running at any time, and the depth of the call stack is at most $|\Gamma|$.
This argument is similar to the one used in the proof of Lemma~\ref{lem:heap_search_complexity} for \heapS, but in \heapS the recursive calls occur when updating the data structures, which is precisely what leads to generating programs of too high cost.

\begin{definition}[\ecoS without bucketing]\label{def:eco_no_bucket}
\emph{\ecoS without bucketing} is the algorithm described above, in which each queue $\queue_X$ is implemented as a standard priority queue (a binary heap) rather than a bucket queue.
\end{definition}
Its delay is $O(\log t)$ (where $t$ is the number of programs output so far), as for \heapS and \beeS.
As we will see in the experiments, it already solves more tasks than \heapS and \beeS.

\subsection{An example by hand}
\label{sec:eco_example}

We consider the grammar and associated costs defined in Figure~\ref{fig:running_example_grammar}.
Algorithm~\ref{algo:minimal_programs} finds $\mincost(\ntstr) = 1.1$ and $\mincost(\ntint) = 1.8$.
During initialisation, we add the following cost tuples:
\[
(r_1, \emptyset), (r_2, \emptyset), (r_3, (0)), (r_4, (0,0)) \text{ to } \queue_\ntstr \quad \text{ and } \quad (r_5, \emptyset), (r_6, \emptyset), (r_7, (0,0)) \text{ to } \queue_\ntint.
\]
At this point, the queues are as follows, with costs indicated below cost tuples:
\[
\queue_\ntstr = \set{\underbrace{(r_1, \emptyset)}_{1.1},
\underbrace{(r_2, \emptyset)}_{2.0},
\underbrace{(r_3, (0))}_{6.2},
\underbrace{(r_4, (0,0))}_{7.5}}
\qquad
\queue_\ntint = \set{\underbrace{(r_5, \emptyset)}_{1.8},
\underbrace{(r_6, \emptyset)}_{3.3},
\underbrace{(r_7, (0,0))}_{8.9}}
\]
and $\indextocost_\ntstr = [1.1]$, $\indextocost_\ntint = [1.8]$.
Let us analyse the first calls of the main loop:
\begin{enumerate}
	\item $\textsc{Output}(\ntstr,0)$: we pop $(r_1, \emptyset)$ from $\queue_\ntstr$ and add $H$ to $\generated_\ntstr[1.1]$.
	The next tuple has cost $2.0$, so we exit the \textbf{while} loop and set $\indextocost_\ntstr[1] = 2.0$.
	\item $\textsc{Output}(\ntstr,1)$: we pop $(r_2, \emptyset)$ from $\queue_\ntstr$ and add $W$ to $\generated_\ntstr[2.0]$. We set $\indextocost_\ntstr[2] = 6.2$.
	\item $\textsc{Output}(\ntstr,2)$: we pop $(r_3, (0))$ from $\queue_\ntstr$.
	Since $\done_\ntint = 0$, line 14 triggers the recursive call $\textsc{Output}(\ntint,0)$. During this call, we pop $(r_5, \emptyset)$ from $\queue_\ntint$, add $\text{var}$ to $\generated_\ntint[1.8]$, and set $\indextocost_\ntint[1] = 3.3$ by peeking at $\queue_\ntint$: this is frugal expansion, the program $1$ is not generated.
	Back to the first call, we add $\text{cast}(\text{var})$ to $\generated_\ntstr[6.2]$.
	We then update the data structure: we push $(r_3, (1))$ to $\queue_\ntstr$ with cost $6.2 + 3.3 - 1.8 = 7.7$.
	The next tuple in $\queue_\ntstr$ is $(r_4, (0,0))$ of cost $7.5$, so we set $\indextocost_\ntstr[3] = 7.5$.
	\item $\textsc{Output}(\ntstr,3)$: we pop $(r_4, (0,0))$ from $\queue_\ntstr$.
	Both arguments refer to $\generated_\ntstr[1.1] = \set{H}$, which is already computed, so we add $\text{concat}(H,H)$ to $\generated_\ntstr[7.5]$.
	We then push $(r_4, (1,0))$ and $(r_4, (0,1))$ to $\queue_\ntstr$, both with cost $7.5 + 2.0 - 1.1 = 8.4$.
	The next tuple is $(r_3, (1))$ of cost $7.7$, and we set $\indextocost_\ntstr[4] = 7.7$.
	\item $\textsc{Output}(\ntstr,4)$: we pop $(r_3, (1))$ from $\queue_\ntstr$.
	Since $\done_\ntint = 1$, line 14 triggers $\textsc{Output}(\ntint,1)$, which pops $(r_6, \emptyset)$, adds $1$ to $\generated_\ntint[3.3]$, and sets $\indextocost_\ntint[2] = 8.9$ by peeking at $\queue_\ntint$.
	We add $\text{cast}(1)$ to $\generated_\ntstr[7.7]$, and push $(r_3, (2))$ with cost $7.7 + 8.9 - 3.3 = 13.3$.
	The next tuples have cost $8.4$, and we set $\indextocost_\ntstr[5] = 8.4$.
	\item $\textsc{Output}(\ntstr,5)$: we pop the two tuples of cost $8.4$, which generate $\text{concat}(W,H)$ and $\text{concat}(H,W)$.
\end{enumerate}
The programs from $\ntstr$ are output in the same order as with \heapS (Section~\ref{sec:heap_search}).
Compare step 5 above with step 6 of the example for \heapS: after outputting $\text{cast}(1)$, \heapS generated $\text{add}(\text{var},\text{var})$ of cost $8.9$ in order to insert $\text{cast}(\text{add}(\text{var},\text{var}))$, while \ecoS only read the cost $8.9$ in $\queue_\ntint$ and pushed one cost tuple.
Compared with \beeS, no popped cost tuple is empty.

\subsection{Correctness}
\label{sec:eco_correctness}

We now prove that \ecoS outputs all programs from $S$, in order of increasing costs, each exactly once.
In this section we do not use observational equivalence (lines 17 and 18 of Algorithm~\ref{algo:beap_search_main_loop} are ignored), and the queues can be any priority queues.
Recall that $c_X^0 < c_X^1 < \cdots$ are the distinct costs of programs from $X$.
We say that a cost tuple $(r, n)$ for a rule $r : X \rightarrow f(X_1,\dots,X_k)$ is \emph{valid} if $c_{X_i}^{n_i}$ is defined for every $i$, and we define its \emph{true cost} as $\cost(r) + \sum_{i} c_{X_i}^{n_i}$ and the set of programs it represents as the set of $f(\program_1,\dots,\program_k)$ with $\cost(\program_i) = c_{X_i}^{n_i}$.
Every program from $X$ is represented by exactly one valid cost tuple: the grammar being deterministic, the rule $r$ is unique, and so are the indices $n_i$.

\begin{lemma}[Invariants of \ecoS]
\label{lem:eco_invariants}
At any time during the execution of \ecoS, for every non-terminal $X$:
\begin{enumerate}
	\item[(1)] every index $\ell$ for which $\indextocost_X[\ell]$ is defined satisfies $\indextocost_X[\ell] = c_X^\ell$, and $\indextocost_X$ has at most $\done_X + 1$ elements;
	\item[(2)] for every $\ell < \done_X$, $\generated_X[c_X^\ell]$ is the set of all programs from $X$ of cost $c_X^\ell$;
	\item[(3)] every cost tuple $(r, (n_1,\dots,n_k))$ in $\queue_X$ is valid, its value is its true cost, and $n_i \leq \done_{X_i}$ for every $i$; moreover every valid cost tuple $(r, m)$ that has not been popped yet has some ancestor $(r, m')$ in $\queue_X$, where $m' \leq m$ componentwise;
	\item[(4)] cost tuples are popped from $\queue_X$ by non-decreasing costs, and any cost tuple pushed into $\queue_X$ has a cost strictly larger than the last popped one.
\end{enumerate}
Moreover, whenever $\textsc{Output}(X,\ell)$ is called we have $\ell = \done_X$, and when it returns, $\generated_X[c_X^\ell]$ is the set of all programs from $X$ of cost $c_X^\ell$; in particular this set is not empty.
\end{lemma}

\begin{proof}
We first observe that costs strictly decrease along the call stack: a call $\textsc{Output}(X,\ell)$ only triggers calls $\textsc{Output}(X_i, n_i)$ for the arguments of a cost tuple of cost $\indextocost_X[\ell]$, and $\indextocost_{X_i}[n_i] < \indextocost_X[\ell]$ since $\cost(r) > 0$.
Moreover the condition in line 13 ensures that calls are made with $n_i = \done_{X_i}$, and the main loop calls $\textsc{Output}(S, \ell)$ after $\textsc{Output}(S, \ell - 1)$ has returned, so $\ell = \done_S$ as well.

We prove the invariants by induction on the number of operations performed.
At initialisation, $\queue_X$ contains $(r, (0,\dots,0))$ for each rule $r$ from $X$, whose value is its true cost since $\mincost(X_i) = c_{X_i}^0$.
Every valid cost tuple $(r, m)$ has $(r, (0,\dots,0))$ as ancestor, and $\indextocost_X[0] = \mincost(X) = c_X^0$, so (1) to (4) hold.

Consider a call $\textsc{Output}(X, \ell)$ with $\ell = \done_X$, and let $c = \indextocost_X[\ell] = c_X^\ell$ by (1).
By (3) and (4), the cost tuples of $\queue_X$ that have cost $c$ are exactly the valid cost tuples of true cost $c$, which represent all programs from $X$ of cost $c$, because every valid tuple of true cost $c$ that has not been popped has an ancestor in $\queue_X$ of cost at most $c$, and all costs in the queue are at least $c$.
When such a tuple $t = (r, n)$ is popped, we have $n_i \leq \done_{X_i}$ by (3), so the recursive calls of line 14 ensure that $n_i < \done_{X_i}$ for every $i$ (the recursive calls are on non-terminals whose current cost is smaller than $c$, so we can apply the induction hypothesis to them), so by (2) the product in line 15 ranges over exactly the programs represented by $t$.
Since every set $\generated_{X_i}[c_{X_i}^{n_i}]$ is non-empty, at least one program is generated.
The children of $t$ pushed in lines 21 to 26 are valid, their value is their true cost by (1), their indices satisfy (3) since $\indextocost_{X_i}[n_i + 1]$ is defined only if $n_i + 1 \leq \done_{X_i}$ by (1) and $\done$ never decreases, and this cost is strictly larger than $c$ because $c_{X_i}^{n_i + 1} > c_{X_i}^{n_i}$; hence they are not popped during the current call, (3) is maintained (every valid tuple whose ancestor in the queue was $t$ now has an ancestor among the children of $t$), and so is (4).
Note that $\indextocost_{X_i}[n_i + 1]$ is defined in line 24 if and only if $c_{X_i}^{n_i + 1}$ exists: if $X_i \neq X$ this follows from the frugal expansion step at the end of $\textsc{Output}(X_i, n_i)$, and if $X_i = X$ we have $n_i < \ell$ so $n_i + 1 \leq \ell$.
At the end of the \textbf{while} loop all programs from $X$ of cost $c$ have been generated, which gives (2) for $\done_X = \ell + 1$.
Finally, after the \textbf{while} loop, all cost tuples in $\queue_X$ have cost larger than $c$, and by (3) the cheapest one has cost $c_X^{\ell + 1}$ (it is the cheapest valid cost tuple not popped yet, and it represents at least one program).
Moreover, cost tuples are only pushed into $\queue_X$ during calls $\textsc{Output}(X, \cdot)$, so $\queue_X$ does not change until the next call $\textsc{Output}(X, \ell+1)$.
Hence the value $\indextocost_X[\ell+1]$ computed by the frugal expansion in line 29 is equal to $c_X^{\ell+1}$, and remains correct afterwards, which proves (1).
\end{proof}

\begin{theorem}[Best-first ordering and completeness]\label{thm:best_first}
\ecoS outputs every program generated from $S$ exactly once, and in non-decreasing cost order: if $\program$ is output before $\program'$, then $\cost(\program) \leq \cost(\program')$.
\end{theorem}

\begin{proof}
The main loop calls $\textsc{Output}(S, 0), \textsc{Output}(S, 1), \dots$, and by Lemma~\ref{lem:eco_invariants} the call $\textsc{Output}(S, \ell)$ outputs exactly the programs from $S$ of cost $c_S^\ell$.
Note that only $\queue_S$ is consulted to decide what to output next: the queues of the other non-terminals are only used through recursive calls, so there is no need to compare the minima of several queues.
Since costs $c_S^0 < c_S^1 < \cdots$ are increasing, programs are output in non-decreasing cost order.
Each program from $S$ is represented by a unique valid cost tuple, which is popped exactly once, so the program is output exactly once.
Finally, let $\program$ be a program from $S$.
There are finitely many programs of cost at most $\cost(\program)$, so $\cost(\program) = c_S^\ell$ for some $\ell$, and $\program$ is output during the call $\textsc{Output}(S,\ell)$.
\end{proof}

\subsection{Bucketing and constant delay}
\label{sec:bucketing}

To introduce bucket queues, we need to state and prove some properties of the costs of programs induced by pre-generation cost functions.
For a program $\program$ from $X$, if there exist programs from $X$ of cost larger than $\cost(\program)$, we write $\costsucc(\program)$ for the smallest such cost, that is, the cost of the next programs of $\program$.
Note that $\program$ may have several next programs, but they all have the same cost.

We first prove that consecutive costs are close to each other, which needs a few definitions.
A non-terminal $Y$ is \emph{recursive} if some program from $Y$ has a node other than the root labelled by $Y$.
For a recursive $Y$, a \emph{$Y$-context} is obtained from a program from $Y$ by replacing the sub-program rooted at a node other than the root labelled by $Y$ by a hole of type $Y$.
Its cost is the sum of the costs of the rules that remain, and plugging a program $\programbis$ from $Y$ in the hole gives a program from $Y$ of cost $\cost(\programbis)$ plus the cost of the context.
Since $Y$ is recursive, there are $Y$-contexts, and their costs are positive since the rule at the root remains.
For a fixed program $\programbis_0$ from $Y$, the cost of a context $C$ is the cost of the program obtained by plugging $\programbis_0$ in $C$ minus $\cost(\programbis_0)$, so there are finitely many costs of $Y$-contexts below any bound.
Hence the minimal cost $K_Y > 0$ of a $Y$-context exists.
Let $K$ be the maximum of $K_Y$ over all recursive non-terminals $Y$ ($K = 0$ if there is none).
Finally, let $\mathcal{F}$ be the set of programs generated from some non-terminal in which no node is labelled by a recursive non-terminal.
In such a program, a non-terminal cannot appear twice on a path from the root to a leaf, so the height is at most $|\Gamma|$ and $\mathcal{F}$ is finite.

\begin{lemma}\label{lemma:cost_difference}
Let
\[
M = \max\left(K,\ \max\set{\costsucc(\programbis) - \cost(\programbis) : \programbis \in \mathcal{F},\ \costsucc(\programbis) \text{ is defined}}\right),
\]
where the maximum of the empty set is $0$.
For any program $\program$ such that $\costsucc(\program)$ is defined, we have $\costsucc(\program) - \cost(\program) \leq M$.
\end{lemma}
\begin{proof}
If $\program \in \mathcal{F}$ this holds by definition of $M$.
Otherwise, some node of $\program$ is labelled by a recursive non-terminal $Y$; let $\programbis$ be the sub-program rooted at this node, which is a program from $Y$.
Take a $Y$-context $C$ of cost $K_Y$, plug $\programbis$ in its hole, and replace $\programbis$ by the result in $\program$.
Since the result is a program from $Y$, this gives a program from the same non-terminal as $\program$, of cost $\cost(\program) + K_Y$.
Since $K_Y > 0$, we get $\cost(\program) < \costsucc(\program) \leq \cost(\program) + K_Y \leq \cost(\program) + M$.
\end{proof}

The constant $M$ only depends on the grammar and the costs of the rules, not on the number of programs output so far.
It can be computed before the search: $K$ is obtained with a shortest path computation similar to Algorithm~\ref{algo:minimal_programs}, and $\mathcal{F}$ is a finite set.
For the bucket queues below, any upper bound on $M$ can be used instead of $M$.

Lemma~\ref{lemma:cost_difference} bounds how far apart consecutive costs can be.
Lemma~\ref{lemma:bound_cost} lifts this to the queues: no matter how long \ecoS has been running, all cost tuples waiting in a queue $\queue_X$ lie within a window of constant width.
For a non-terminal $X$, let $D_X$ be the difference between the largest and the smallest costs of the initial cost tuples $(r, (0,\dots,0))$ for the rules $r$ from $X$, and $D = \max_X D_X$.

\begin{lemma}\label{lemma:bound_cost}
Let $M' = \max(M, D)$.
At any time during the execution of \ecoS, for any non-terminal $X$, all cost tuples $t$ in the queue $\queue_X$ satisfy:
\begin{equation*}
	\cost(t) - \min_{t' \in \queue_X} \cost(t') \leq M'.
\end{equation*}
\end{lemma}
\begin{proof}
Let $t$ be in $\queue_X$, and let $\mu$ be the minimal cost in $\queue_X$.
If $t$ is an initial cost tuple, then $\cost(t) \leq \mincost(X) + D_X$, and $\mu \geq \mincost(X)$ since all programs from $X$ cost at least $\mincost(X)$, so $\cost(t) - \mu \leq D$.
Otherwise, $t = (r, n')$ was pushed as a child of a popped cost tuple $(r, n)$, where $n'$ and $n$ only differ in some index $i$, with $n'_i = n_i + 1$.
By Lemma~\ref{lem:eco_invariants}, $\cost(t) - \cost((r,n)) = c_{X_i}^{n_i + 1} - c_{X_i}^{n_i}$, which is at most $M$ by Lemma~\ref{lemma:cost_difference} applied to a program from $X_i$ of cost $c_{X_i}^{n_i}$.
Since $(r, n)$ was popped before all cost tuples currently in $\queue_X$, we have $\cost((r,n)) \leq \mu$ by item (4) of Lemma~\ref{lem:eco_invariants}, so $\cost(t) - \mu \leq M$.
\end{proof}

Let us now assume that the cost function takes integer values: $\cost(r) \in \mathbb{N}_{> 0}$ for every derivation rule $r$, so that $\cost(\program) \in \mathbb{N}_{>0}$ for every program.
This also covers costs that are multiples of a fixed quantum $\delta > 0$, by dividing all costs by $\delta$.
In practice the costs are real numbers (for instance $-\log \Prob(r)$), and they are first rounded to multiples of a small quantum; in our experiments the quantum is $10^{-5}$ (Section~\ref{sec:experiments}).
The cost function in Figure~\ref{fig:running_example_grammar} becomes integer after multiplying all costs by $10$.

Thanks to Lemma~\ref{lemma:bound_cost}, there are at most $M'+1$ different costs in a queue at any given time, so we can implement the queues as \emph{bucket queues}, a classical data structure introduced by~\citet{Dial:1969} for shortest paths with integer weights (see also~\cite{Denardo.Fox:1979,Ahuja.Mehlhorn.ea:1990,Cherkassky.Goldberg.ea:1999,Thorup:2000} and the textbook~\cite{Edelkamp.Schrodl:2012} for their use in heuristic search).
Concretely, a bucket queue is an array of $M'+1$ lists, the list at index $j$ containing the cost tuples whose cost is congruent to $j$ modulo $M'+1$.
We keep track of the index $j$ of the list that contains cost tuples of minimal cost.
To pop a cost tuple, we take one from the $j$\textsuperscript{th} list.
If the list at index $j$ is empty, we increment $j$ modulo $M'+1$ until we find a non-empty list; by Lemma~\ref{lemma:bound_cost} this takes at most $M'+1$ steps.
To push a cost tuple whose cost is $d$ plus the current minimal cost, we add it to the list at index $(j+d) \mod (M'+1)$.
With lists implemented as linked lists, pushing takes constant time and popping takes time $O(M')$, which is a constant.

When $M'$ is large, scanning empty lists can be slow.
In our implementation we use a multi-level bucket queue~\cite{Denardo.Fox:1979,Cherkassky.Goldberg.ea:1999}: the window of width $M'$ is split into $k$ cells, and a cell that receives cost tuples of different costs is itself split into $k$ cells, and so on.
The parameter $k$ is what we call the \emph{bucket size} in the experiments.
The number of levels is at most $\lceil \log_k (M'+1) \rceil$, so every operation takes time $O(k \log_k M')$, which is still independent of the number of programs output.

\begin{theorem}[Constant delay]\label{thm:constant_delay}
Assume that costs are integers.
Then \ecoS implemented with bucket queues and without observational equivalence has constant delay: the amount of computation between two consecutive programs generated (from any non-terminal) is bounded by a constant that depends on the grammar and the cost function, but not on the number of programs already generated.
In particular, the first $t$ programs are generated in time $O(t)$.
\end{theorem}

\begin{proof}
Let $m$ be the maximal arity of a symbol, and $B$ the time of an operation on a bucket queue, which is constant by Lemma~\ref{lemma:bound_cost} and the discussion above.
The test ``$t'$ is not in $\queue_X$'' can be done in constant time with a hash set storing the content of each queue (in expectation), or avoided altogether by pushing each cost tuple from a single canonical parent, for instance the parent obtained by decrementing the last non-zero index; in that case every valid cost tuple still has an ancestor in the queue, which is all we use in Lemma~\ref{lem:eco_invariants}.
Each program is constructed in time $O(m)$ (line 16), and we now bound the work between two consecutive constructions.

We call \emph{elementary operation} the start of a call to $\textsc{Output}$ (lines 6 to 8), a pop (lines 9 to 11), a test of line 13, an iteration of the update loop (lines 21 to 26), and the end of a call (lines 27 to 29).
Each of them takes time at most $B$.
We show that there are at most $(m+2)|\Gamma|$ elementary operations before the first program is generated, and at most $2m + 2 + (m+2)|\Gamma|$ between two consecutive generated programs.

For a call $\textsc{Output}(X, \ell)$ of cost $c$, let $h$ be the number of non-terminals $Y$ whose current cost $\indextocost_Y[\done_Y]$ is at most $c$.
For a recursive call $\textsc{Output}(Y, j)$ of cost $c' < c$ made by this call, the corresponding number is at most $h - 1$: the current cost of $X$ is $c > c'$, and the non-terminals of current cost larger than $c$ are not modified by the call.
Hence $h \leq |\Gamma|$.
We show by induction on $h$ that during the call $\textsc{Output}(X, \ell)$, including its recursive calls:
(a) at least one program is generated;
(b) there are at most $(m+2)h$ elementary operations before the first generated program;
(c) there are at most $m+1$ elementary operations after the last generated program;
(d) there are at most $2m + 2 + (m+2)h$ elementary operations between two consecutive generated programs.
By Lemma~\ref{lem:eco_invariants}, the first popped cost tuple has cost $c$ and represents at least one program, which gives (a).
Before its first program, the call starts, pops a cost tuple, performs at most $m$ tests of line 13, and performs the operations of the first recursive call before its first program, if there is one; by induction this gives at most $2 + m + (m+2)(h-1) = (m+2)h$ operations, hence (b).
After its last program, the call performs at most $m$ iterations of the update loop and ends, hence (c).
For (d), consider two consecutive generated programs.
If both are generated by the same recursive call, we use the induction hypothesis.
If the first one is the last program of a recursive call and the second one is the first program of the next recursive call, there are at most $(m+1) + m + (m+2)(h-1)$ operations in between, by (b) and (c) for these calls.
If the first one is the last program of a popped cost tuple and the second one comes from the next popped cost tuple, there are at most $m$ iterations of the update loop, one pop, $m$ tests of line 13, and the operations of a recursive call before its first program, so at most $2m + 1 + (m+2)(h-1)$ operations.
In all cases the bound of (d) holds.
Finally, in the main loop, between the last program of $\textsc{Output}(S, \ell)$ and the first program of $\textsc{Output}(S, \ell + 1)$ there are at most $(m+1) + (m+2)|\Gamma|$ elementary operations.

Hence the time between two consecutive generated programs is $O((m + 2)|\Gamma| \cdot B + m)$, a constant.
\end{proof}

The bound on the number of elementary operations does not use the assumption that costs are integers: this assumption is only used so that each queue operation takes constant time.
Theorem~\ref{thm:constant_delay} counts all generated programs, whatever their non-terminal.
As discussed in Section~\ref{sec:delay}, this is the natural measure for program synthesis from examples, since every generated program is evaluated on the examples.

\paragraph*{Choice of the window.}
The bucket queues require an upper bound on $M'$ before the search starts, and Lemma~\ref{lemma:cost_difference} gives a constructive bound, computed from the grammar only.
If an upper bound is not available or is too pessimistic, a simple alternative is to start with a small window and, whenever a cost tuple falls outside of it, to double the window and rebuild the queue; since $M'$ is a constant, this happens at most $\log_2 M'$ times, so this does not change the asymptotic delay after a constant number of steps.

\paragraph*{Memory.}
Every bottom-up algorithm that uses observational equivalence stores all the programs it generates, since new programs are compared to them.
With shared sub-programs (Section~\ref{sec:delay}) each program takes constant space, and besides these programs \ecoS stores the cost tuples in the queues.
Each popped cost tuple pushes at most $m$ children, and each popped cost tuple represents at least one program (Lemma~\ref{lem:eco_invariants}), so the number of cost tuples ever pushed is at most the number of rules plus $m$ times the number of generated programs.
Hence the memory used by \ecoS is linear in the number of programs generated.
In comparison, \heapS stores up to $m$ explicitly constructed programs per generated program (the programs in $\seen$), many of which have a cost much larger than the current one, and \beeS may push cost tuples that represent no program, so that the number of cost tuples in its queue is not bounded in terms of the number of generated programs.

\subsection{Using only some of the ingredients}
\label{sec:eco_ablation}

The ingredients of \ecoS are not independent, and it is interesting to see what one gets by adding them one by one to \beeS or \heapS.
\begin{itemize}
	\item \emph{Per non-terminal structures without frugal expansion.} Adding per non-terminal structures to \beeS removes the type check, but creates the problem discussed in Section~\ref{sec:frugal_expansion}: the cost $\indextocost_{X_i}[n_i + 1]$ may be unknown when pushing children. The only alternative to frugal expansion is to compute it by generating the corresponding programs, which is what \heapS does, with the drawbacks explained above. Frugal expansion has no counterpart in \beeS, because it is not needed there.
	Generating these programs at the moment their cost is needed does not even terminate on grammars with cycles: with the rules $W \rightarrow k(U)$ and $U \rightarrow u_1 \mid q(W)$ of Section~\ref{sec:heap_search}, generating the next programs from $W$ requires the next programs from $U$, which require the next programs from $W$, and so on. The \emph{eager} variant of \ecoS that we evaluate in Section~\ref{sec:experiments} therefore generates the next programs of a non-terminal (or of a tuple of non-terminals) when frugal expansion reads their cost, but only after the current call on the start non-terminal is finished, and without doing the same for the programs it generates in this way.
	\item \emph{Bucket queues without per non-terminal structures.} Lemmas~\ref{lemma:cost_difference} and~\ref{lemma:bound_cost} also hold for \beeS, with a single list of costs, so bucket queues can be used in \beeS, and each cost tuple is then processed in constant time. However, the cost tuples popped by \beeS may represent no program at all, and the number of such tuples between two generated programs is not bounded (Section~\ref{sec:bee_search}), so this does not give constant delay. The property that every cost tuple represents at least one program, which is the key to the proof of Theorem~\ref{thm:constant_delay}, comes from per non-terminal structures.
	\item \emph{Bucket queues in \heapS.} \heapS orders programs rather than cost tuples, and the programs in $\heap_X$ may have costs much larger than the current one, so Lemma~\ref{lemma:bound_cost} does not apply to it.
\end{itemize}
In this sense, the constant-delay guarantee is a property of the combination: it needs per non-terminal structures (so that no cost tuple is empty), cost tuples with frugal expansion (so that the queues only contain costs in a bounded window), and bucket queues (so that queue operations take constant time).
In the experiments, \beeS, \ecoS without bucketing, and \ecoS form a sequence where ingredients are added one at a time, and we also evaluate \beeS with bucket queues and \ecoS with eager expansion.

\subsection{Observational equivalence and post-generation cost functions}
\label{sec:eco_extensions}

\paragraph*{Observational equivalence.}
With observational equivalence, a program is discarded right after it is constructed if it is equivalent to a program already in $\generated_X$ (lines 17 and 18 of Algorithm~\ref{algo:beap_search_main_loop}).
Such a program is never added to $\generated_X$, so it is never used as an argument of a larger program.
Note that the cost tuples, the lists $\indextocost_X$ and the queues do not depend on which programs are discarded: the update part (lines 21 to 26) is executed for every popped cost tuple, whatever the result of the generation part.
Hence the sequence of popped cost tuples is the same as without observational equivalence, Lemma~\ref{lem:eco_invariants} holds with the sets $\generated_X[c]$ replaced by subsets of the programs from $X$ of cost $c$ (possibly empty), and \ecoS remains best-first and complete up to observational equivalence in the sense of Section~2.4: the argument given there applies, since every program built from kept sub-programs is constructed and then kept or discarded.
In contrast with \heapS (Section~\ref{sec:heap_search_oe}), no special care is needed, because a cost tuple is expanded whether or not its programs are kept.
The bounded window of Lemma~\ref{lemma:bound_cost} also holds, so bucket queues can still be used.
What is lost is the guarantee that every cost tuple represents at least one program that is kept: a set $\generated_X[c]$ may become empty, and the cost tuples that use it generate nothing.
Therefore, with observational equivalence, the bound of Theorem~\ref{thm:constant_delay} holds for the work per cost tuple and per program constructed (before the equivalence check), but not per program kept.
This is unavoidable in general: if all programs from $X$ of large cost are equivalent to small ones, they are all discarded, and nothing can be output for a long time.

\paragraph*{Post-generation cost functions.}
Post-generation cost functions compute the cost of a program after it has been constructed and evaluated, for instance using a neural network that takes the outputs of the program on the examples as input, as in \bustle~\cite{Odena.Shi.ea:2021} and \crossbeam~\cite{Shi.Dai.ea:2022}.
\citet{Ameen.Lelis:2023} consider additive and penalizing post-generation cost functions: when a program $\program = f(\program_1,\dots,\program_k)$ is generated by a rule $r$, its cost is
\[
w'(\program) = \cost(r) + \sum_{i = 1}^k w'(\program_i) + \pi(\program), \quad \text{ with } \pi(\program) \geq 0,
\]
where the penalty $\pi(\program)$ is only known once $\program$ has been evaluated.
\beeS handles such cost functions by inserting the new costs in its list of costs at the right position, which may require rebuilding the queue (a ``heapify'' operation); this costs time linear in the size of the queue in the worst case.

\ecoS can be adapted to such cost functions when the penalty is bounded, that is, $\pi(\program) \leq \Pi$ for a constant $\Pi$, as in the cost function $w_{\text{Bustle}}$ used by~\citet{Ameen.Lelis:2023}, where $\pi(\program)$ is an integer between $0$ and $5$.
The cost of a cost tuple is now $\cost(r) + \sum_i \indextocost_{X_i}[n_i]$, which is the cost of its programs before penalties.
A program $\program$ generated from a cost tuple of cost $c$ is evaluated, and if $\pi(\program) > 0$ it is not added to $\generated_X[c]$ but to an additional bucket queue $\pending_X$ with value $c + \pi(\program)$.
In $\textsc{Output}(X,\ell)$, the programs of $\pending_X$ of value $\indextocost_X[\ell]$ are added to $\generated_X[\indextocost_X[\ell]]$, and the frugal expansion step (line 29) uses the minimum of the cheapest cost in $\queue_X$ and the cheapest value in $\pending_X$.
Since penalties are non-negative, a program is never added to $\pending_X$ with a value smaller than the current cost, so the arguments of Lemma~\ref{lem:eco_invariants} and Theorem~\ref{thm:best_first} go through: \ecoS outputs programs in order of increasing $w'$.
As with observational equivalence, the list $\indextocost_X$ may now contain costs for which no program exists (all programs of a cost tuple may have been penalised), and these indices are simply skipped.
The values in $\pending_X$ lie in a window of width $\Pi$ above the current cost.
For the cost tuples, the proof of Lemma~\ref{lemma:cost_difference} adapts as follows: choose the node labelled by a recursive non-terminal $Y$ closest to the root, so that it has at most $|\Gamma|$ ancestors, and plug its sub-program in a $Y$-context repeated $j$ times; each of the at most $|\Gamma|$ ancestors may see its penalty change by at most $\Pi$, so for $j$ large enough (depending on $\Pi$, $|\Gamma|$ and the cost of the context, but not on the program) the cost increases by a bounded and positive amount.
Hence the windows remain bounded, and bucket queues can still be used.
With integer costs, at most a constant number of consecutive indices can be empty, because consecutive costs of actual programs differ by a bounded amount, so the delay remains constant.
We only sketch this extension and did not implement it: our experiments use pre-generation cost functions only.

When the penalty is unbounded, as for the cost function $w_u$ of~\citet{Ameen.Lelis:2023} which uses $-\log_2$ of a probability, the windows are not bounded anymore, and our constant-delay argument does not apply.
\ecoS can still be used with standard priority queues for $\pending_X$ and $\queue_X$, and remains best-first by the same argument, but we have no guarantee on its delay and we did not evaluate it in this setting.
For general post-generation cost functions, \beeS remains the reference best-first algorithm.
 
\section{Experiments}
\label{sec:experiments}
\begin{figure*}[ht]
    \includegraphics[width=0.9\columnwidth]{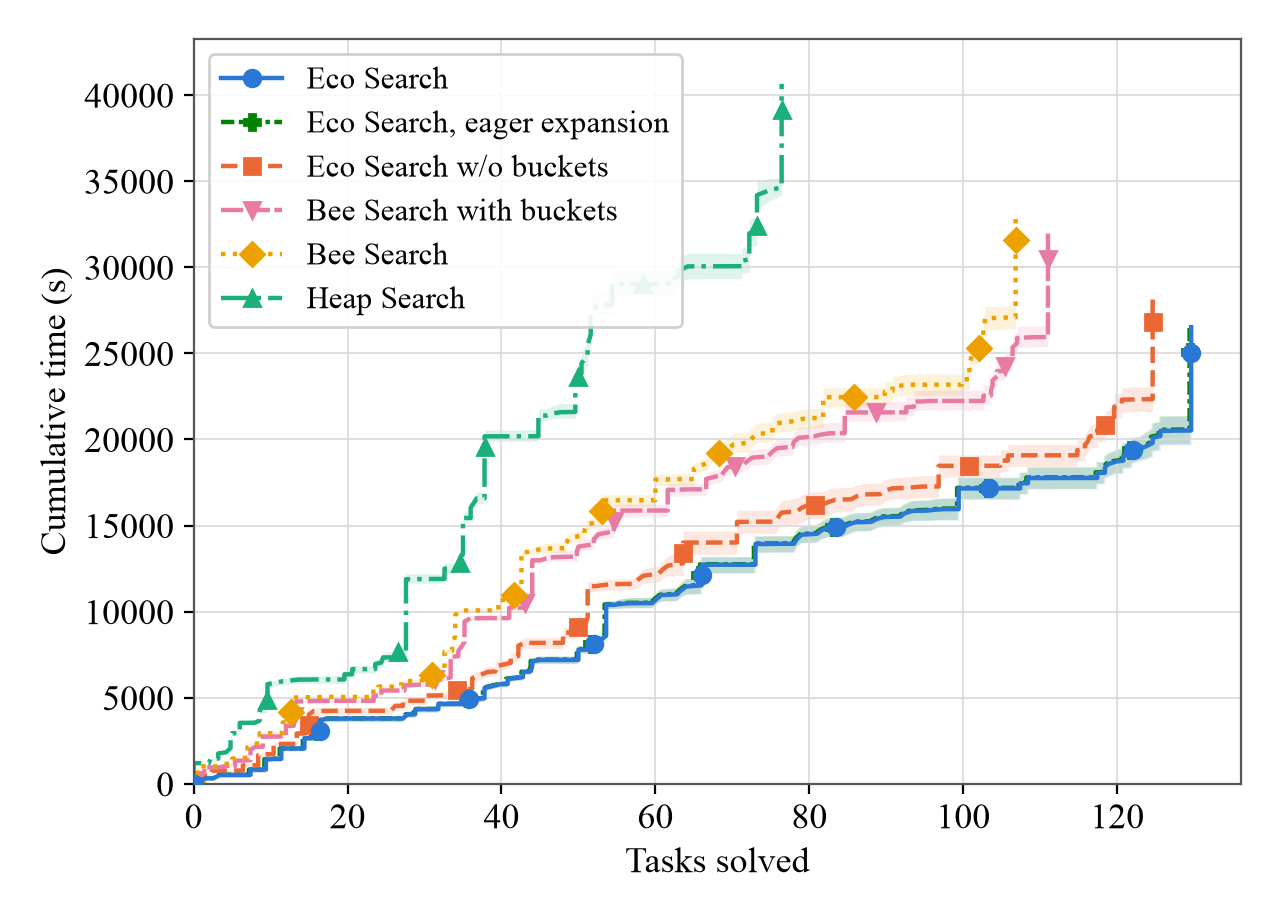}
    \caption{String manipulations from SyGuS using FlashFill's DSL ($205$ tasks). The tasks are taken in a fixed order; a point gives, for the first $n$ tasks, the number of tasks solved (\emph{x-axis}) and the total time spent on them in seconds (\emph{y-axis}), a failure counting for the timeout of $300$ seconds. Lower and to the right is better. Mean over five seeds, the shaded areas are $95\%$ bootstrap confidence intervals. The curves of \ecoS and of its variant with eager expansion almost coincide.}
    \label{fig:string_results}
\end{figure*}

\begin{figure*}[ht]
    \includegraphics[width=0.9\columnwidth]{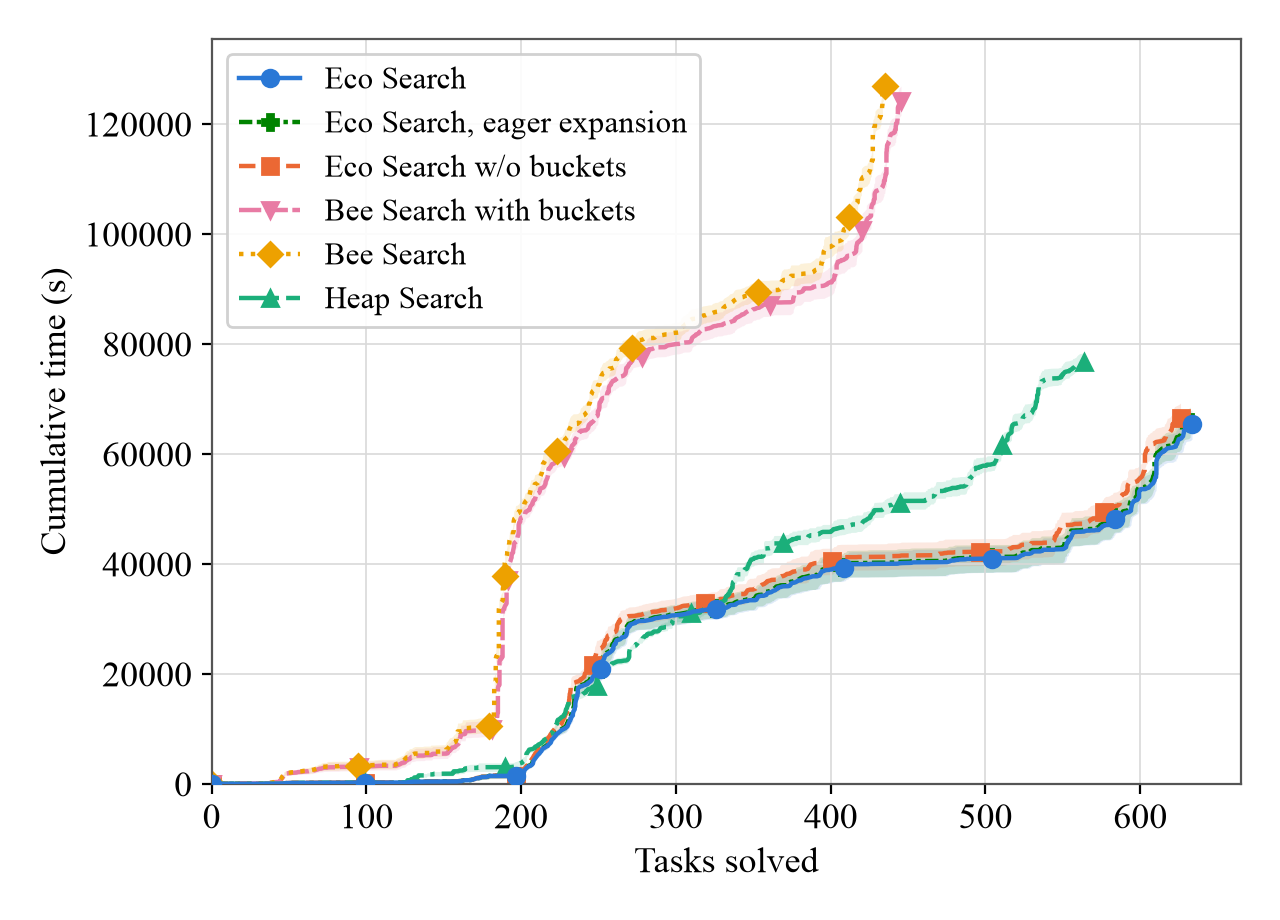}
    \caption{Integer list manipulations using DeepCoder's DSL ($792$ tasks). Same conventions as in Figure~\ref{fig:string_results}.}
    \label{fig:deepcoder_results}
\end{figure*}

\begin{figure*}[ht]
    \centering
    \begin{subfigure}{0.8\textwidth}
        \includegraphics[width=\columnwidth]{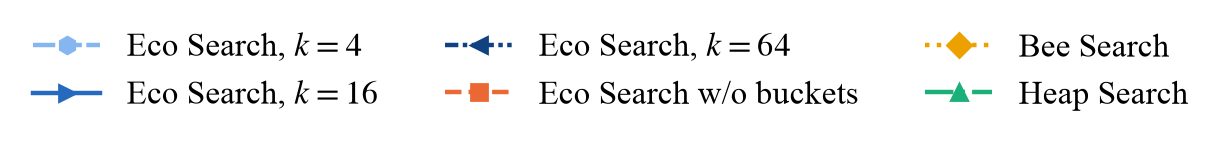}
    \end{subfigure}

    \begin{subfigure}[t]{0.48\textwidth}
        \includegraphics[width=\columnwidth]{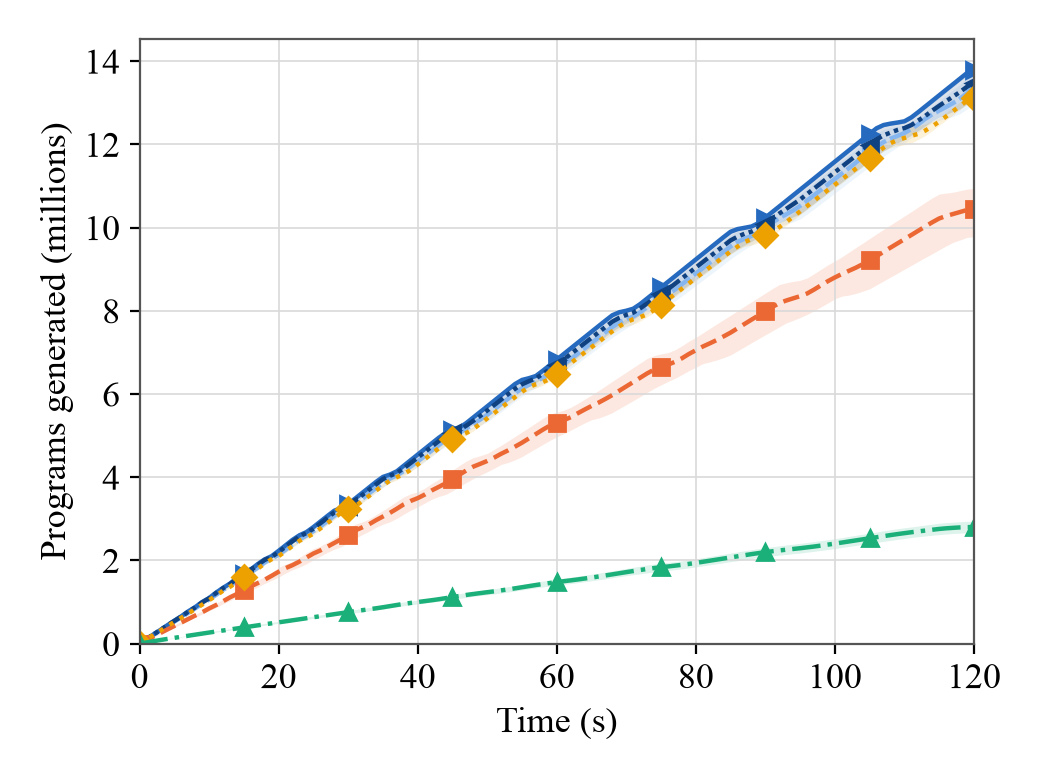}
        \caption{Throughput for $D_4$ ($11$ derivation rules)}
    \end{subfigure}\hfill
    \begin{subfigure}[t]{0.48\textwidth}
        \includegraphics[width=\columnwidth]{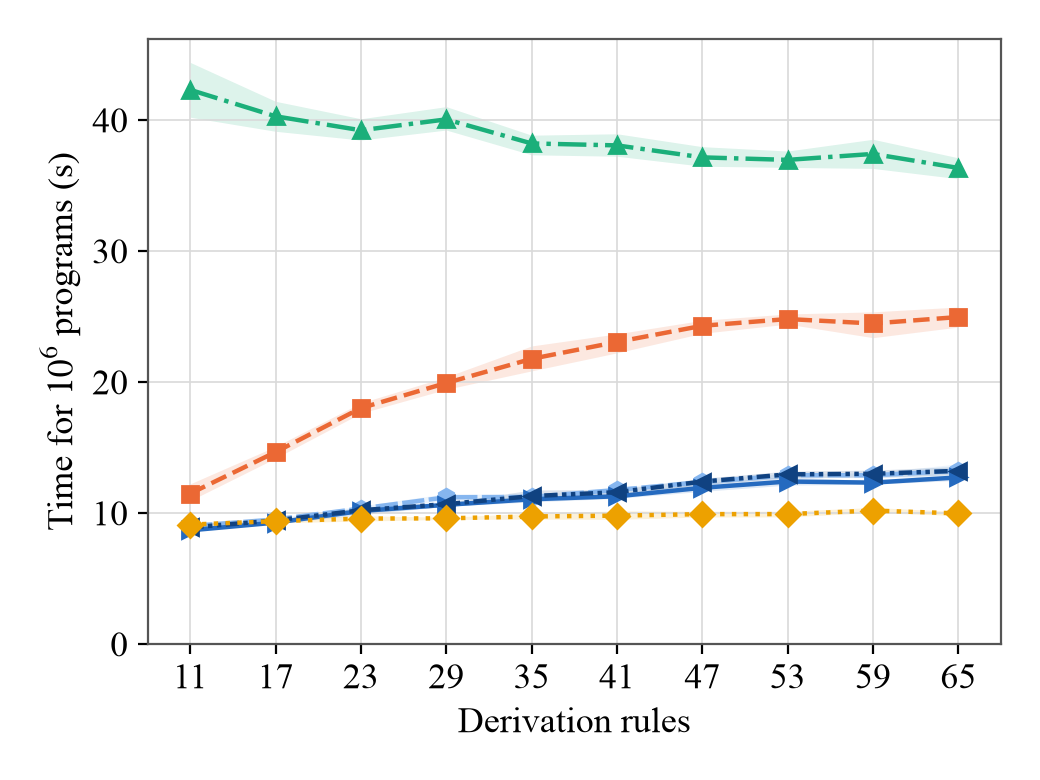}
        \caption{Scaling law for $D_k$ ($3k-1$ derivation rules)}
    \end{subfigure}
    \caption{Scaling against the number of derivation rules. Mean over five random cost functions, the shaded areas are $95\%$ bootstrap confidence intervals.}
    \label{fig:scaling}
\end{figure*}

\begin{figure*}[ht]
    \centering
    \begin{subfigure}[t]{0.48\textwidth}
        \includegraphics[width=\columnwidth]{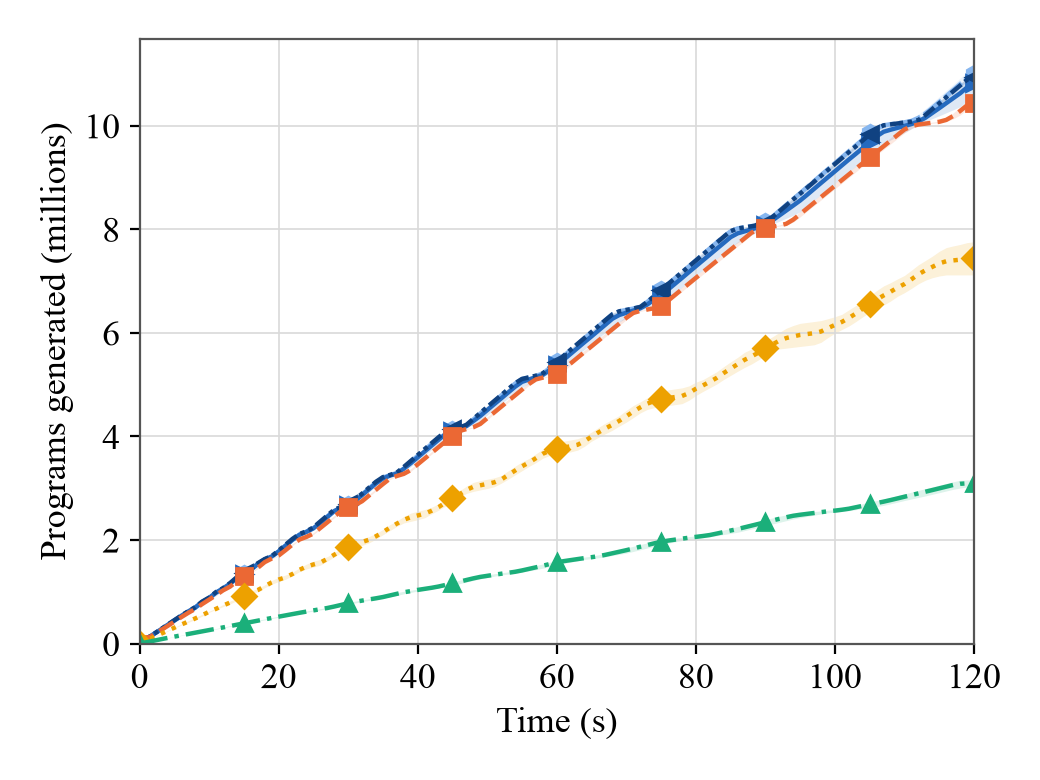}
        \caption{Throughput for $N_4$ ($4$ non-terminals)}
    \end{subfigure}\hfill
    \begin{subfigure}[t]{0.48\textwidth}
        \includegraphics[width=\columnwidth]{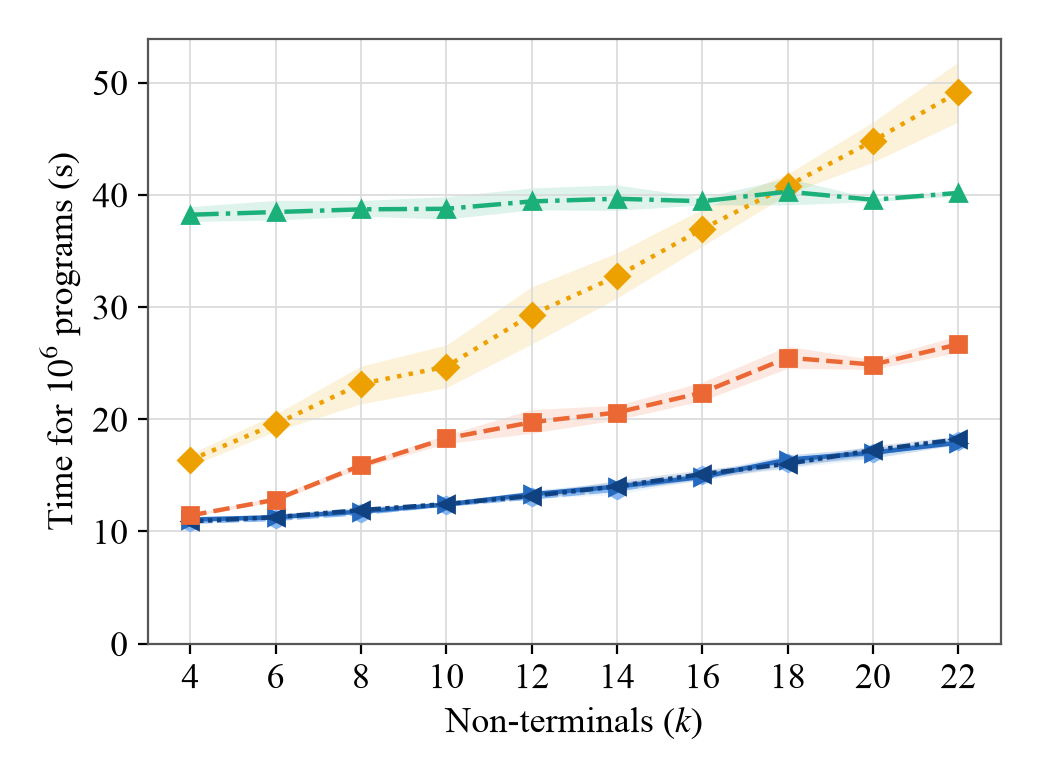}
        \caption{Scaling law for $N_k$ ($k$ non-terminals)}
    \end{subfigure}
    \caption{Scaling against the number of non-terminals. Same conventions as in Figure~\ref{fig:scaling}.}
    \label{fig:scaling_nonterminals}
\end{figure*}

\begin{figure*}[ht]
    \centering
    \begin{subfigure}[t]{0.48\textwidth}
        \includegraphics[width=\columnwidth]{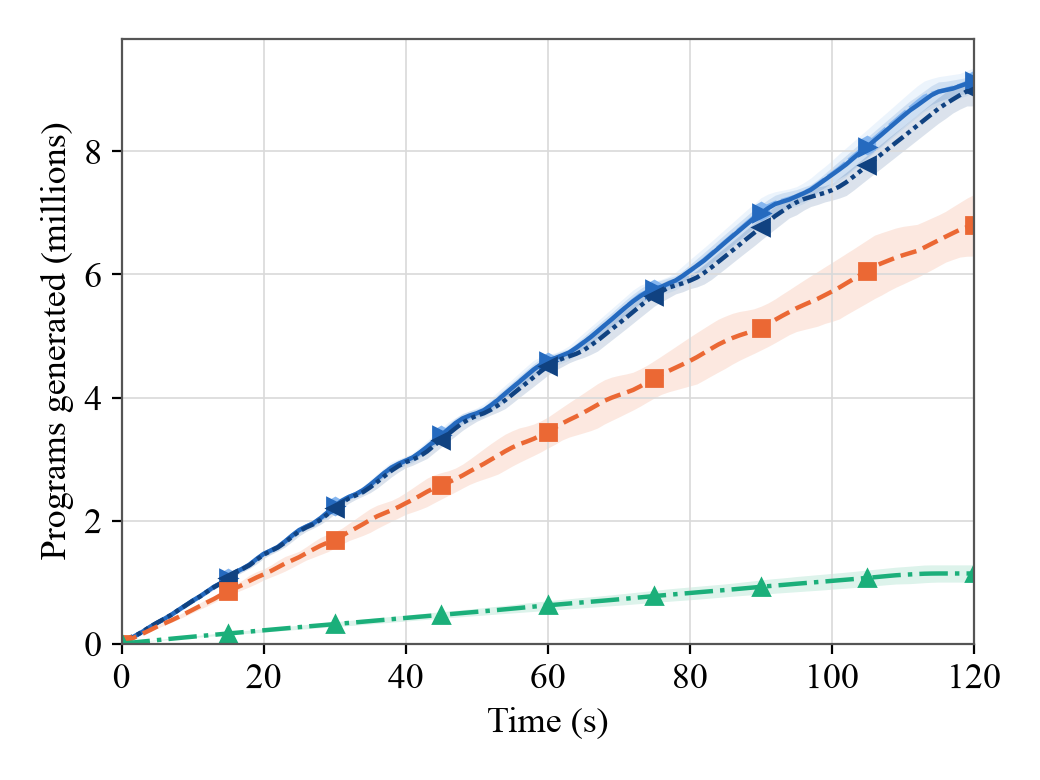}
        \caption{Throughput for $R_4$ ($4$ non-terminals)}
        \label{fig:scaling_dist}
    \end{subfigure}\hfill
    \begin{subfigure}[t]{0.48\textwidth}
        \includegraphics[width=\columnwidth]{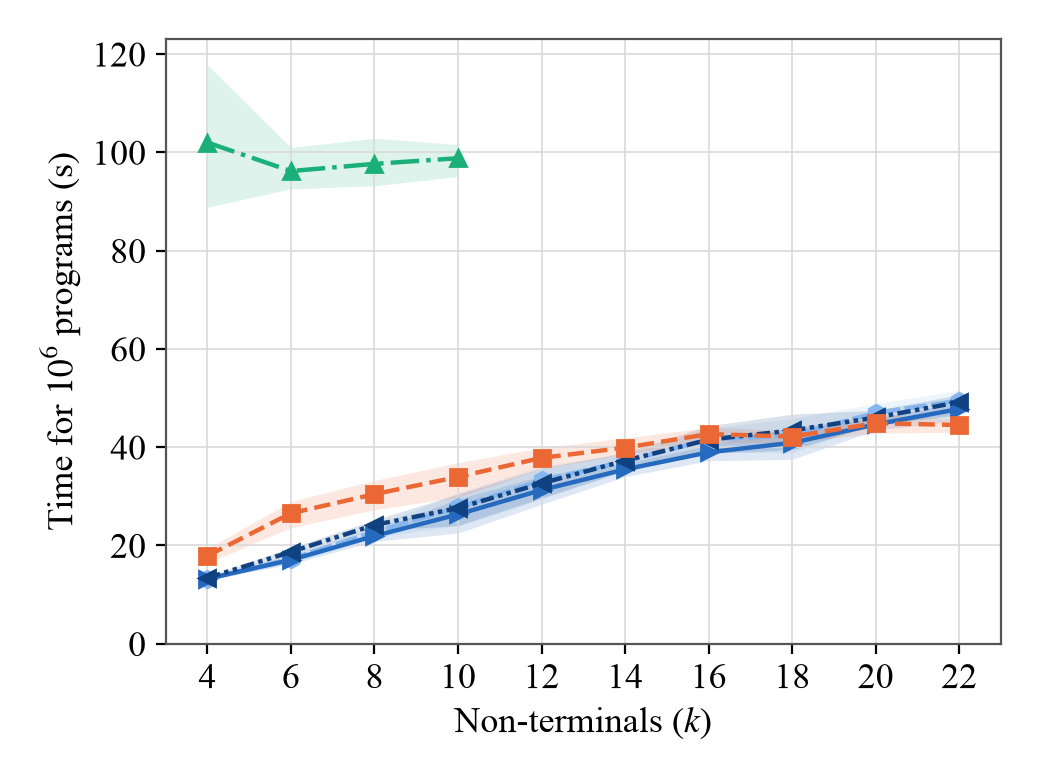}
        \caption{Scaling law for $R_k$ ($k$ non-terminals)}
    \end{subfigure}
    \caption{Scaling against the number of non-terminals, with rules that combine several non-terminals. Same conventions as in Figure~\ref{fig:scaling}. \beeS generates only about $10^4$ programs in $120$ seconds on $R_4$ and is not shown. \heapS was run up to $k = 10$.}
    \label{fig:scaling_distance}
\end{figure*}

\begin{table}[ht]
\centering
\caption{Tasks solved and peak memory per task. Tasks solved: mean over the five seeds, with the minimum and the maximum in parentheses. Peak memory: largest resident memory of the process solving a task minus its memory at the start, mean and maximum over all tasks and seeds, in GB.}
\label{tab:solve_results}
\begin{tabular}{lcccc}
\toprule
 & \multicolumn{2}{c}{FlashFill ($205$ tasks)} & \multicolumn{2}{c}{DeepCoder ($792$ tasks)} \\
Algorithm & Tasks solved & Memory & Tasks solved & Memory \\
\midrule
\ecoS                         & $\mathbf{129.6}$ ($127$--$133$) & $1.13$ / $4.6$ & $\mathbf{633.4}$ ($620$--$646$) & $0.74$ / $3.8$ \\
\ecoS with eager expansion    & $129.4$ ($126$--$133$) & $1.13$ / $4.7$ & $632.2$ ($620$--$645$) & $0.74$ / $3.7$ \\
\ecoS without bucketing       & $124.6$ ($118$--$128$) & $1.05$ / $4.3$ & $626.0$ ($615$--$638$) & $0.73$ / $3.6$ \\
\beeS with bucket queues      & $111.0$ ($104$--$114$) & $0.75$ / $3.1$ & $445.4$ ($423$--$465$) & $0.33$ / $0.9$ \\
\beeS                         & $106.8$ ($104$--$110$) & $0.66$ / $2.9$ & $435.4$ ($419$--$451$) & $0.30$ / $0.8$ \\
\heapS                        & $76.4$ ($72$--$78$)    & $2.56$ / $5.4$ & $563.8$ ($556$--$570$) & $0.87$ / $3.6$ \\
\bottomrule
\end{tabular}
\end{table}

To check whether the theoretical properties of \ecoS lead to practical gains, we ask the following questions:
\begin{itemize}
    \item[\textbf{Q1:}] Does \ecoS improve over existing best-first search algorithms on program synthesis tasks?
    \item[\textbf{Q2:}] How does the performance of these algorithms scale with the complexity of the grammar?
\end{itemize}

\paragraph*{Datasets.}
We consider two classic domains: string manipulations and integer list manipulations.
For string manipulations we use the same setting as in \beeS~\cite{Ameen.Lelis:2023}: FlashFill's 205 tasks from SyGuS.
The grammar has $3$ non-terminals, one per type (strings, integers and Booleans), and $27$ or $28$ derivation rules depending on the number of inputs.
For integer list manipulations we use the $792$ test tasks of the DeepCoder~\cite{Balog.Gaunt.ea:2017} dataset, whose solutions are programs of length $1$ to $5$.
As in~\cite{Fijalkow.Lagarde.ea:2022}, a non-terminal of this grammar is a type together with the rule of the parent node and the position of the argument, so the grammar has $48$ non-terminals and between $1320$ and $1383$ derivation rules, depending on the type of the task.
The timeout is $300$ seconds per task.

\paragraph*{Cost functions.}
The cost functions are the same for all algorithms, following~\cite{Fijalkow.Lagarde.ea:2022}.
For each domain and each random seed, we generate a synthetic training set of $1000$ programs sampled from the uniform grammar, with two sets of input output examples each, and we train a neural network on it for $10$ epochs.
For each task, the network takes the input output examples and outputs a probability for each derivation rule, which defines a probabilistic grammar, and the cost of a rule $r$ is $-\log \Prob(r)$ (Section~\ref{sec:background}).
The network is trained to minimise the negative log-likelihood of the derivation rules of the training programs.
This is where randomness comes from: different seeds give different training sets and different networks, hence different cost functions on the same test tasks.
The search algorithms themselves are deterministic.
For two seeds, the DeepCoder training set contains no program of the type of some test tasks ($3$ and $2$ tasks); these tasks count as failures.

\paragraph*{Implementation.}
All algorithms are implemented in Python, in the DeepSynth framework.
The code and the results are available at \url{https://github.com/SynthesisLab/DeepSynth/tree/jair_rerun}, with the seeds, the cost functions and all other experimental details.
The experiments were run on the Grid'5000 testbed, on nodes with two Intel Xeon Gold 6130 CPUs ($16$ cores each, at $2.1$ GHz) and $192$ GB of memory.
All runs are sequential (one core per run), and up to $12$ runs share a node.
The memory of each run is limited to $16$ GB ($14$ GB for the scaling experiments); no run reached this limit.
The neural networks were trained on GPUs, and everything else runs on CPUs.
We report the mean over five seeds ($1$, $48$, $849$, $35132$ and $98465$ for the cost functions of Section~\ref{sec:experiments_synthesis}, and $1$ to $5$ for the random cost functions of Section~\ref{sec:experiments_scaling}), with $95\%$ confidence intervals computed by bootstrapping over the seeds.

\paragraph*{Algorithms.}
We compare \ecoS against the two state of the art best-first search algorithms: \heapS and \beeS.
Since they are all bottom-up algorithms they all use observational equivalence (pruning programs with same outputs on all input examples).
The implementation of \heapS follows the corrected version of Section~\ref{sec:heap_search}.
None of the algorithms have hyperparameters except for the rounding off procedure for costs.
For \beeS we follow the original implementation and round costs (which are negative logarithms of probabilities) to multiples of $10^{-2}$.
For \ecoS we need integer costs (Theorem~\ref{thm:constant_delay}): the cost $-\log \Prob(r)$ of each rule is rounded down to a multiple of $10^{-5}$, and then multiplied by $10^5$ to obtain an integer.
In other words, two rules whose costs differ by less than $10^{-5}$ may be given the same cost.
This is a much finer rounding than the one used for \beeS.
In our implementation of \ecoS, the queues of cost tuples are multi-level bucket queues (Section~\ref{sec:bucketing}), one for each tuple of argument non-terminals, shared by the rules that have the same arguments.
Their width is estimated from the grammar before the search starts, from the spread of the costs of the initial cost tuples, and doubled if a cost tuple falls outside the window, as explained in the paragraph ``Choice of the window'' of Section~\ref{sec:bucketing}.
Their only parameter is the bucket size $k$, the number of cells per level.
By default we use $k = 20$, and we also experiment with $k \in \set{4, 16, 64}$.
On top of them, the queue of a non-terminal $X$ contains at most one entry per rule from $X$, and is a binary heap: its operations take time logarithmic in the number of rules, which does not depend on the number of programs output.
For comparison, we also consider \ecoS without bucketing (Definition~\ref{def:eco_no_bucket}).
In terms of ingredients (Table~\ref{tab:algorithm_comparison}), \beeS, \ecoS without bucketing, and \ecoS add one ingredient at a time: per non-terminal structures with frugal expansion, and then bucket queues.
We also evaluate two variants discussed in Section~\ref{sec:eco_ablation}: \beeS with bucket queues, where the queue of \beeS is a bucket queue with $k = 20$, and \ecoS with eager expansion instead of frugal expansion.

\subsection{Does \ecoS improve over existing best-first search algorithms on program synthesis tasks?}
\label{sec:experiments_synthesis}

We run all algorithms on our benchmarks, with a timeout of $300$ seconds per task.
Figures~\ref{fig:string_results} and~\ref{fig:deepcoder_results} show the cumulative time with respect to the number of tasks solved, and Table~\ref{tab:solve_results} gives the number of tasks solved and the peak memory.

\paragraph*{Comparison with \heapS and \beeS.}
\ecoS solves the most tasks on both domains, for every seed: $129.6$ of the $205$ FlashFill tasks on average, and $633.4$ of the $792$ DeepCoder tasks.
Compared with \beeS, it solves $21\%$ more tasks on FlashFill and $45\%$ more on DeepCoder, and it reaches the number of tasks solved by \beeS $1.6$ and $3.2$ times faster, respectively.
Compared with \heapS, it solves $70\%$ more tasks on FlashFill and $12\%$ more on DeepCoder, and it reaches the number of tasks solved by \heapS $2.5$ and $1.7$ times faster.
We compared the time spent on each task (mean over the seeds, a failure counting for $300$ seconds) with two-sided Wilcoxon signed-rank tests: \ecoS is faster than \beeS and \heapS on both domains, with $p < 10^{-20}$, except against \heapS on DeepCoder, where $p = 0.02$.

The two previous algorithms do not rank in the same order on the two domains: \heapS is last on FlashFill, but third on DeepCoder, ahead of \beeS.
The difference between \beeS and the other algorithms is much larger on DeepCoder, whose grammar has $48$ non-terminals, than on FlashFill, whose grammar has $3$.
This is consistent with the next experiment, where \beeS scales badly with the number of non-terminals: since \beeS has a single queue for all non-terminals, many of the cost tuples it pops do not represent any program.
On FlashFill, \heapS also uses the most memory.
Recall that \heapS constructs and expands the programs that are equivalent to earlier ones (Section~\ref{sec:heap_search_oe}), so it benefits less from observational equivalence than \beeS and \ecoS.

\paragraph*{The ingredients of \ecoS.}
Per non-terminal structures with frugal expansion bring most of the gain: \ecoS without bucketing solves $17.8$ more tasks than \beeS on FlashFill, and $190.6$ more on DeepCoder.
Bucket queues add $5.0$ tasks on FlashFill and $7.4$ on DeepCoder.
On FlashFill, \ecoS is faster than its variant without bucketing ($p < 10^{-16}$), and reaches the number of tasks solved by this variant $1.1$ times faster.
On DeepCoder, the time per task differs in both directions: \ecoS without bucketing is faster on more tasks, by a few milliseconds, but solves fewer tasks.
Adding bucket queues to \beeS helps a little ($4.2$ more tasks on FlashFill and $10.0$ on DeepCoder), but it remains far from \ecoS.
This is what Section~\ref{sec:eco_ablation} predicts: with bucket queues each cost tuple is processed in constant time, but \beeS still pops cost tuples that represent no program.
Finally, replacing frugal expansion by eager expansion makes almost no difference: the variant solves the same tasks, up to one or two, and it is slower on most tasks, by $1\%$ to $3\%$ in the median.
In our measurements on these grammars, eager expansion generates programs earlier than frugal expansion, but only a few dozen more in total.
Frugal expansion is needed for the guarantee of constant delay, but its practical effect is small here.

\paragraph*{Memory.}
No run reached the limit of $16$ GB, so memory is not what limits any of the algorithms with a timeout of $300$ seconds: with more memory, they would not solve more tasks.
The largest peak is $5.4$ GB (\heapS on FlashFill).
The three versions of \ecoS use about the same memory, and \beeS uses the least, at most $0.9$ GB on DeepCoder, where it is also the slowest.
From a theoretical point of view, \ecoS uses memory linear in the number of programs it generates (Section~\ref{sec:bucketing}), which is the same order as the memory needed to store the programs themselves, while \heapS also stores many programs of large cost that it does not output, and \beeS may store cost tuples that represent no program.

\paragraph*{Summary}
\begin{quote}
    \ecoS solves the most tasks on both domains, $21\%$ and $45\%$ more than \beeS and $70\%$ and $12\%$ more than \heapS, and reaches the number of tasks they solve $1.6$ to $3.2$ times faster.
    Per non-terminal structures with frugal expansion bring most of the gain, and bucket queues add a few tasks.
\end{quote}

\subsection{How does the performance of these algorithms scale with the complexity of the grammar?}
\label{sec:experiments_scaling}
The goal of these experiments is to understand how well our algorithms perform on more complicated grammars.
However there is no agreed upon definition of ``grammar complexity'' as different measures of complexity can be used.
Information-theoretic measures such as the Kolmogorov complexity of the grammar are not computable, and are not related in an obvious way to the running time of search algorithms, so we use simple structural parameters instead.
A bad proxy for grammar complexity is the number of programs it generates: it is in most cases infinite, and as a function of depth it grows extremely fast hence cannot be accurately compared.
We identify three parameters:
\begin{itemize}
    \item The number of derivation rules;
    \item The number of non-terminals;
    \item The number of non-terminals when rules combine several non-terminals.
\end{itemize}
In our experiments, we measure the performance of our algorithms for pure enumeration: the programs are not evaluated on input examples, enumeration continues for a fixed amount of time.
For each parameter, we created parametric grammars, where $x$ is the input variable:
\begin{itemize}
    \item The grammar $D_k$ has $3k-1$ derivation rules.
    It uses a single non-terminal $S$.
    The primitives are: $k-1$ primitives $f_i$ of arity $2$, $k-1$ primitives $g_i$ of arity $1$, and $k$ constants $h_i$ (arity $0$).
    The derivation rules are $S \to h_1$, $S \to x$, and, for each $i \in [2,k]$:
    \[
    S \to f_i(S,S) \quad;\quad S \to g_i(S) \quad;\quad S \to h_i
    \]

    \item The grammar $N_k$ has $k$ non-terminals, called $S_1,\dots,S_k$, with $S_1$ initial, and $3k-1$ derivation rules.
    The primitives are: $2k-2$ primitives $f_i,g_i$ of arity $1$ and $k$ constants $h_i$ (arity $0$).
    The derivation rules are $S_1 \to h_1$, $S_1 \to x$, and, for each $i \in [2,k]$:
    \[
        S_1 \to f_i(S_i) \quad;\quad S_i \to g_i(S_1) \quad;\quad S_i \to h_i
    \]

    \item The grammar $R_k$ has $k$ non-terminals, called $S_1,\dots,S_k$, with $S_1$ initial, and $4k-3$ derivation rules.
    The primitives are: $k-2$ primitives $f_i$ of arity $3$, $k-1$ primitives $g_i$ of arity $2$, $k-1$ primitives $h_i$ of arity $1$, and $k$ constants $c_i$ (arity $0$).
    The derivation rules are $S_1 \to c_1$, $S_1 \to x$, $S_i \to f_i(S_{i-1}, S_i, S_{i+1})$ for each $i \in [2,k-1]$, and, for each $i \in [2,k]$:
    \begin{align*}
        S_1 \to h_i(S_i) & \quad;\quad S_i \to g_i(S_1, S_i) \quad;\quad S_i \to c_i
    \end{align*}
    Like $N_k$, it has $k$ non-terminals, but its rules of arity $2$ and $3$ combine programs from different non-terminals.
\end{itemize}
\noindent For each grammar and each of five seeds, the probabilities of the rules are drawn at random, and the cost of a rule $r$ is $-\log \Prob(r)$.
For each of the three parameters, we consider two scenarios:
\begin{enumerate}
    \item \emph{Throughput}: For a fixed grammar ($k = 4$), how many programs are enumerated as a function of time.
    \item \emph{Scaling law}: For a range of values of the parameter, how long it takes to enumerate one million programs. We run each algorithm for $120$ seconds and report the time per million programs.
\end{enumerate}

We plot the results for the three parameters and both scenarios on Figures~\ref{fig:scaling}--\ref{fig:scaling_distance}.

First, we look at the evolution with the number of derivation rules.
\ecoS and \beeS are the fastest algorithms.
With a single non-terminal, every cost tuple popped by \beeS represents at least one program, so this is the most favourable case for \beeS.
\beeS is almost insensitive to the number of rules ($9.1$ to $10.0$ seconds per million programs from $11$ to $65$ rules), while \ecoS slows down slightly ($8.7$ to $12.7$ seconds with $k = 16$), so that \beeS is slightly faster from $17$ rules on.
Without bucketing, \ecoS is about twice as slow ($24.9$ seconds for $65$ rules).
\heapS is three to five times slower than \ecoS, and does not depend on the number of rules.

Second, looking at the number of non-terminals, \ecoS is the fastest algorithm for all values of $k$, followed by \ecoS without bucketing.
\beeS slows down the most with the number of non-terminals: from $4$ to $16$ non-terminals, its time per million programs is multiplied by $2.2$ ($16.4$ to $36.9$ seconds), against $1.3$ for \ecoS ($11.1$ to $14.9$ seconds with $k = 16$) and $2.0$ for \ecoS without bucketing.
From $18$ non-terminals on, \beeS is slower than \heapS, whose time does not depend on $k$ (about $39$ seconds).

Finally, $R_k$ is the hardest family of grammars.
\beeS generates only about $10^4$ programs in $120$ seconds on $R_4$, that is, about $12\,000$ seconds per million programs, so it is not shown.
\heapS takes about $100$ seconds per million programs, four to eight times more than \ecoS; we ran it up to $R_{10}$.
The time of \ecoS grows with $k$, from $13.2$ seconds for $R_4$ to $47.8$ seconds for $R_{22}$ (with $k = 16$).
Without bucketing, \ecoS is slower up to $R_{14}$, and then as fast or slightly faster ($44.5$ seconds for $R_{22}$).
We do not have an explanation for this.

The throughput curves also show how the algorithms slow down over time.
For all algorithms except \heapS, the number of programs generated in the last $20$ seconds differs from the number generated in the first $20$ seconds by less than about $10\%$, in either direction.
\heapS slows down: it generates $22\%$ fewer programs in the last $20$ seconds on $D_4$, and $44\%$ fewer on $R_4$.
Over $120$ seconds, the logarithmic factor in the delay of \beeS and of \ecoS without bucketing is not visible.

\paragraph*{Summary}
\begin{quote}
    \ecoS scales better than \beeS with the number of non-terminals, and is much faster on grammars whose rules combine several non-terminals.
    \beeS is slightly faster on grammars with a single non-terminal and many rules.
    \heapS is insensitive to the size of the grammar, but two to eight times slower than \ecoS.
    \ecoS is robust to the choice of bucket size.
\end{quote}

\subsection{Limitations of the experimental evaluation}
\label{sec:experiments_limitations}

\paragraph*{Implementation language.}
All algorithms are implemented in Python, in the same framework, and share the same code for everything except the search itself (grammars, cost functions, evaluation of programs, observational equivalence).
Python is much slower than languages such as C++ or Rust, but this overhead affects all algorithms in a similar way, and our comparison is about the search algorithms rather than absolute running times.
The implementation of \heapS is the one of its authors~\cite{Fijalkow.Lagarde.ea:2022}, some of whom are authors of the present paper, with the correction of Section~\ref{sec:heap_search}.
\beeS is a re-implementation following the pseudocode and the choices of~\citet{Ameen.Lelis:2023}, including the rounding of costs; we cannot exclude that their implementation contains optimisations that ours lacks.
A comparison with the original implementation of \beeS would require porting it to our benchmarks and cost functions, which we leave for future work.
We note however that the scaling experiments of Section~\ref{sec:experiments_scaling}, which do not involve program evaluation, are consistent with the analysis of Sections~\ref{sec:predecessor_algorithms} and~\ref{sec:eco_search}: for instance, \beeS slows down with the number of non-terminals.

\paragraph*{Cost functions.}
The predictions of the neural networks depend on the type of the task, but in practice not on its input output examples: the networks act as priors over the derivation rules, learned from the training sets.
These priors are close to the uniform grammar, and better only on DeepCoder: the mean cost of the known solutions is between $22.4$ and $25.1$ on FlashFill ($123$ tasks with a known solution) depending on the seed, against $22.5$ for the uniform grammar, and between $18.8$ and $20.6$ on DeepCoder, against $21.2$.
Since all algorithms use the same cost functions, this does not affect the comparison, but better cost functions would increase the number of tasks solved by all algorithms.

\paragraph*{Parallelism.}
All experiments are sequential, and we discuss parallel versions of best-first search in Section~\ref{sec:conclusions}.
 
\section{Related Works}
\label{sec:related_works}
Combinatorial search for program synthesis has been an active area~\cite{Alur.Singh.ea:2018}, and a powerful tool in combination with neural approaches~\cite{Chaudhuri.Ellis.ea:2021}.
In particular, cost-guided combinatorial search provides a natural way of combining statistical or neural predictions with search~\cite{Menon.Tamuz.ea:2013,Balog.Gaunt.ea:2017}.

\paragraph*{Best-first search for program synthesis.}
By exploring the space in the exact order induced by the cost function, best-first search algorithms form a natural family of algorithms.
The first best-first search algorithm constructed in the context of cost-guided combinatorial search was an $A^*$ algorithm~\cite{Alur.Radhakrishna.ea:2017}.
\ecoS can be thought of as the unification of \heapS~\cite{Fijalkow.Lagarde.ea:2022} and \beeS~\cite{Ameen.Lelis:2023}, both best-first search bottom-up algorithms.

Best-first search over programs has also been studied in other communities.
In inductive logic programming, Progol~\cite{Muggleton:1995} searches the space of clauses with an $A^*$-like algorithm guided by a compression measure, and Brute~\cite{Cropper.Dumancic:2020} performs a best-first search over programs guided by a loss function computed on the examples, going beyond entailment; Popper~\cite{Cropper.Morel:2021} instead generates programs of increasing size and prunes the search space with constraints learned from failures (see~\cite{Cropper.Dumancic:2022} for a survey).
In automated planning, generalized planning can be cast as the synthesis of programs that solve a family of planning instances: \citet{Segovia-Aguas.Jimenez.ea:2019} compile this problem into classical planning, and \citet{Segovia-Aguas.Jimenez.ea:2021} solve it with a best-first search in the space of programs guided by heuristic functions.
These works use best-first search with heuristics defined for their own setting, and to our knowledge do not study the delay of the enumeration; the ideas developed here (cost tuples, per non-terminal structures, bounded windows) could be relevant whenever the search space is described by a grammar and the heuristic is additive.

\paragraph*{Bucket queues.}
Bucket queues were introduced by~\citet{Dial:1969} to implement Dijkstra's algorithm for graphs with small integer weights, and refined into multi-level buckets and radix heaps~\cite{Denardo.Fox:1979,Ahuja.Mehlhorn.ea:1990,Cherkassky.Goldberg.ea:1999}.
They are also commonly used to implement open lists in heuristic search when costs are integers~\cite{Edelkamp.Schrodl:2012,Burns.Hatem.ea:2012}.
In these settings, the property that makes buckets efficient is that the keys in the queue lie in a window of bounded width, which holds for instance for Dijkstra's algorithm when the weights are bounded.
Our Lemma~\ref{lemma:bound_cost} establishes the analogous property for the cost tuples of \ecoS.

\paragraph*{Bottom-up synthesis and machine learning.}
\ecoS follows the bottom-up paradigm, where larger programs are obtained by composing smaller ones~\cite{Udupa.Raghavan.ea:2013}. Bottom-up algorithms have been successfully combined with machine learning approaches, for instance the PC-Coder~\cite{Zohar.Wolf:2018}, Probe~\cite{Barke.Peleg.ea:2020}, \tfcoder~\cite{Shi.Bieber.ea:2022}, and DreamCoder~\cite{Ellis.Wong.ea:2021}.
In these works, machine learning is used to improve combinatorial search for program synthesis, while \bustle{}~\cite{Odena.Shi.ea:2021} and Execution-Guided Synthesis~\cite{Chen.Liu.ea:2019} use neural models to guide the search process itself.
Alternatively, \crossbeam{}~\cite{Shi.Dai.ea:2022} and LambdaBeam~\cite{Shi.Dai.ea:2023} leverage Reinforcement Learning for this purpose.
More recently, \citet{Li.Zhou.ea:2024} propose Arborist, a bottom-up algorithm that exploits term-rewriting equivalences to prune the search space, achieving strong results on integer-list benchmarks while scaling to larger DSLs.

\paragraph*{Large language models.}
Large language models (LLMs) trained on code can generate programs directly from natural language descriptions~\cite{Chen.Tworek.ea:2021,Austin.Odena.ea:2021,Li.Choi.ea:2022}, and coding agents built on them are now widely used.
They have also been evaluated on programming by example: \citet{Li.Ellis:2024} show that, after fine-tuning, LLMs solve many classical PBE tasks, including string manipulations, but that they are less reliable on tasks that differ from their training distribution.
Symbolic search and LLMs have complementary strengths.
LLMs bring strong prior knowledge and can work with informal specifications, but they offer no guarantee that the solution space is explored systematically, and each sample is expensive.
Symbolic search requires a DSL, but it is exhaustive within it, cheap per program, and its output can be checked against the examples; best-first search moreover guarantees that the most likely programs according to the cost function are tried first.
Several recent works combine the two: LLMs are used to guide enumerative synthesis~\cite{Li.Parsert.ea:2024,Li.Ellis:2024}, and HySynth~\cite{Barke.Gonzalez.ea:2024} learns a probabilistic grammar from programs sampled from an LLM and uses it to guide a bottom-up search.
Such a probabilistic grammar defines a pre-generation cost function, so \ecoS can be used directly in this setting.
LambdaBeam solves many tasks that LLMs cannot solve, thanks to its ability to compose programs.
These results motivate the development of combinatorial search algorithms that are faster and can be guided by neural models.
We did not compare \ecoS with LLM-based systems: our goal is to compare search algorithms with the same cost function, and a comparison with LLMs would mostly measure the quality of the cost function rather than the efficiency of the search.
 
\section{Conclusions}
\label{sec:conclusions}
We introduced \ecoS, a new best-first bottom-up search algorithm for program synthesis.
Its predecessors \heapS~\cite{Fijalkow.Lagarde.ea:2022} and \beeS~\cite{Ameen.Lelis:2023} incur $O(\log t)$ delay between generating the $t$\textsuperscript{th} and $(t+1)$\textsuperscript{th} program.
\ecoS combines the cost tuples of \beeS with the per non-terminal structures of \heapS thanks to a new frugal expansion strategy, and uses bucket queues, which we proved to be applicable because the costs in each queue lie in a window of constant width (Lemma~\ref{lemma:bound_cost}).
We proved that \ecoS is complete, generates programs in best-first order (Theorem~\ref{thm:best_first}), and has constant delay for integer pre-generation cost functions (Theorem~\ref{thm:constant_delay}), which answers a question left open by~\citet{Fijalkow.Lagarde.ea:2022} and~\citet{Ameen.Lelis:2023}.
On two classical domains this improvement translates to practical gains: \ecoS solves $12\%$ to $70\%$ more tasks than previous best-first algorithms, and reaches the number of tasks they solve $1.6$ to $3.2$ times faster.

The present analysis is limited to \emph{pre-generation} cost functions, which are determined by grammar rules independently of any program execution.
\emph{Post-generation} cost functions, where the cost of a program is assessed only after execution as in \bustle~\cite{Odena.Shi.ea:2021} and \crossbeam~\cite{Shi.Dai.ea:2022}, are an important extension.
We sketched in Section~\ref{sec:eco_extensions} how \ecoS can be adapted to additive and penalizing post-generation cost functions with bounded penalties, as in \beeS~\cite{Ameen.Lelis:2023}, while keeping constant delay.
This adaptation remains to be implemented and evaluated, and for unbounded penalties new techniques would be needed, since our bounded window argument does not apply anymore.
For general post-generation cost functions, \beeS remains the reference best-first algorithm.

Parallelism is another way of speeding up search.
\citet{Berger.Fijalkow.ea:2025} study program synthesis on GPUs, and show that best-first search algorithms do not map naturally onto data-parallel hardware because of their sequential priority queue updates.
The bucket queues of \ecoS might help here, since only a window of constant width is active at any time and all cost tuples of the same cost can be processed independently, but we leave the design of a parallel version of \ecoS as an open problem.

Our experiments reveal an important limitation: performance degrades as the grammar grows, and on grammars with a single non-terminal and many rules, \beeS is even slightly faster than \ecoS.
In many practical settings grammars remain compact, but recent work targets much larger DSLs.
For instance, \citet{Hodel:2024} constructs a very large grammar targeting the Abstraction and Reasoning Corpus (ARC)~\cite{Chollet:2019}, and \citet{Singhal.Shroff:2025} approach ARC through concept-level combinatorial search.
Developing best-first algorithms that scale well to such settings, for example by partitioning the grammar or interleaving search with grammar learning, is an important open problem.

\begin{acks}
This work was partially supported by the SAIF project, funded by the
``France 2030'' government investment plan managed by the French National
Research Agency, under the reference ANR-23-PEIA-0006.
Experiments presented in this paper were carried out using the Grid'5000
testbed, supported by a scientific interest group hosted by Inria and including
CNRS, RENATER and several Universities as well as other organizations (see
\url{https://www.grid5000.fr}).
\end{acks}

\printbibliography

\end{document}